\documentclass[letterpaper]{article} 
\usepackage[colorlinks=true,citecolor=brown]{hyperref}
\usepackage{aaai25}  
\usepackage{times}  
\usepackage{helvet}  
\usepackage{courier}  
\usepackage{graphicx} 
\usepackage{natbib}  
\usepackage{caption} 
\usepackage{algorithm}
\usepackage{algorithmic}

\usepackage{amsmath}
\usepackage{amssymb}
\usepackage{mathtools}
\usepackage{amsthm}

\theoremstyle{plain}
\newtheorem{theorem}{Theorem}[section]

\newtheorem{lemma}[theorem]{Lemma}

\theoremstyle{definition}

\theoremstyle{remark}

\usepackage[textsize=tiny]{todonotes}
\usepackage{xspace}

\newcommand{\ourmethod}{RAT\xspace}

\newcommand{\MDP}{RSA-MDP\xspace}
\graphicspath{{figs/}}

\usepackage{amsmath,amsfonts,bm}

\def\eqref#1{equation~\ref{#1}}

\def\1{\bm{1}}

\def\rva{{\mathbf{a}}}

\def\rvs{{\mathbf{s}}}

\DeclareMathAlphabet{\mathsfit}{\encodingdefault}{\sfdefault}{m}{sl}
\SetMathAlphabet{\mathsfit}{bold}{\encodingdefault}{\sfdefault}{bx}{n}

\def\gA{{\mathcal{A}}}
\def\gB{{\mathcal{B}}}

\def\gL{{\mathcal{L}}}
\def\gM{{\mathcal{M}}}

\def\gO{{\mathcal{O}}}
\def\gP{{\mathcal{P}}}

\def\gR{{\mathcal{R}}}
\def\gS{{\mathcal{S}}}

\newcommand{\E}{\mathbb{E}}

\newcommand{\KL}{D_{\mathrm{KL}}}

\newcommand{\normmax}{L^\infty}

\newcommand{\I}{\mathbb{I}}
\newcommand{\at}{\alpha_{t}}
\newcommand{\atnext}{\alpha_{t+1}}

\newcommand{\hata}{\hat{\alpha}}
\newcommand{\hatat}{\hat{\alpha}_{t}}
\newcommand{\hatatnext}{\hat{\alpha}_{t+1}}
\newcommand{\hatatnextnext}{\hat{\alpha}_{t+2}}
\newcommand{\wt}{\omega_{t}}
\newcommand{\wtnext}{\omega_{t+1}}
\newcommand{\outerloss}{J_{\pi}}
\newcommand{\innerloss}{\gL}
\newcommand{\perturbedpi}{\pi_{\nu \circ \alpha}}

\newcommand{\outerlrt}{\beta_{t}}
\newcommand{\innerlrt}{\eta_{t}}

\newcommand{\innerlrtnext}{\eta_{t+1}}
\newcommand{\gradhata}{\nabla_{\hat{\alpha}}}
\newcommand{\grada}{\nabla_{\alpha}}
\newcommand{\gradw}{\nabla_{\omega}}

\newcommand{\expectation}[1]{\mathbb{E}\left[ #1 \right]}
\newcommand{\Bigexpectation}[1]{\mathbb{E}\Big[ #1 \Big]}
\newcommand{\ltwonorm}[1]{\left\| #1 \right\|}

\newcommand{\squaredltwonorm}[1]{\left\| #1 \right\|^2}

\newcommand{\specialcell}[2][c]{\begin{tabular}[#1]{@{}c@{}}#2\end{tabular}}

\usepackage{amsfonts}       
\usepackage{nicefrac}       
\usepackage{microtype}      
\usepackage{xcolor}         
\usepackage{enumitem} 
\usepackage{subfig} 
\usepackage{multirow}
\usepackage{makecell} 
\usepackage{array}
\usepackage{booktabs}

\newcommand{\stdv}[1]{\scalebox{.70}{~$\pm$~#1}}

\newcommand{\pix}{\kern 0.1em}
\newcommand{\pmm}{\kern 0.35em$\pm$\kern 0.35em}
\usepackage{newfloat}
\usepackage{listings}
\DeclareCaptionStyle{ruled}{labelfont=normalfont,labelsep=colon,strut=off} 
\floatstyle{ruled}
\newfloat{listing}{tb}{lst}{}
\floatname{listing}{Listing}
\usepackage{newfloat}
\usepackage{listings}
\DeclareCaptionStyle{ruled}{labelfont=normalfont,labelsep=colon,strut=off} 
\floatstyle{ruled}
\newfloat{listing}{tb}{lst}{}
\floatname{listing}{Listing}
\makeatletter
\renewcommand{\@fnsymbol}[1]{\textasteriskcentered}
\makeatother

\title{RAT: Adversarial Attacks on Deep Reinforcement Agents for Targeted Behaviors}
\author{
    Fengshuo Bai\textsuperscript{\rm 1, \rm 2},
    Runze Liu\textsuperscript{\rm 6},\\
    Yali Du\textsuperscript{\rm 3},
    Ying Wen\textsuperscript{\rm 1,}\footnotemark[1], 
    Yaodong Yang\textsuperscript{\rm 4, \rm 5,}\footnotemark[1], 
}
\affiliations{
    \textsuperscript{\rm 1}Shanghai Jiao Tong University
    \textsuperscript{\rm 2}Zhongguancun Academy
    \textsuperscript{\rm 3}King’s College London \\
    \textsuperscript{\rm 4}Center for AI Safety and Governance, Institute for AI, Peking University \\
    \textsuperscript{\rm 5}State Key Laboratory of General Artificial Intelligence, Peking University\\
    \textsuperscript{\rm 6}Tsinghua Shenzhen International Graduate School, Tsinghua University

}

\usepackage{bibentry}

\begin{document}

\maketitle

\renewcommand{\thefootnote}{\fnsymbol{footnote}}
\footnotetext[1]{Corresponding authors. Email: yaodong.yang@pku.edu.cn, ying.wen@sjtu.edu.cn}

\begin{abstract}
Evaluating deep reinforcement learning (DRL) agents against targeted behavior attacks is critical for assessing their robustness. These attacks aim to manipulate the victim into specific behaviors that align with the attacker’s objectives, often bypassing traditional reward-based defenses. Prior methods have primarily focused on reducing cumulative rewards; however, rewards are typically too generic to capture complex safety requirements effectively. As a result, focusing solely on reward reduction can lead to suboptimal attack strategies, particularly in safety-critical scenarios where more precise behavior manipulation is needed. To address these challenges, we propose RAT, a method designed for universal, targeted behavior attacks. RAT trains an intention policy that is explicitly aligned with human preferences, serving as a precise behavioral target for the adversary. Concurrently, an adversary manipulates the victim's policy to follow this target behavior. To enhance the effectiveness of these attacks, RAT dynamically adjusts the state occupancy measure within the replay buffer, allowing for more controlled and effective behavior manipulation. Our empirical results on robotic simulation tasks demonstrate that RAT outperforms existing adversarial attack algorithms in inducing specific behaviors. Additionally, RAT shows promise in improving agent robustness, leading to more resilient policies. We further validate RAT by guiding Decision Transformer agents to adopt behaviors aligned with human preferences in various MuJoCo tasks, demonstrating its effectiveness across diverse tasks. The supplementary videos are available at \url{https://sites.google.com/view/jj9uxjgmba5lr3g}.
\end{abstract}

%

\section{Introduction}
Reinforcement learning (RL)~\citep{sutton2018reinforcement} combined with deep neural networks (DNN)~\citep{lecun2015deep} shows extraordinary capabilities of allowing agents to master complex behaviors in various domains, including robotic manipulation \cite{wang2023order, picor2023bai}, video games~\cite{zhang2023replay, zhang2024exploiting, wang2023quantifying, wen2024reinforcing}, industrial applications~\citep{xu2023drl, shi2024autonomous, jia2024bench2drive}. However, recent findings~\citep{HuangPGDA17, 3237383.3238064, zhang2020robust, zhang2024a} show that even well-trained DRL agents suffer from vulnerability against test-time attacks, raising concerns in high-risk or safety-critical situations. To understand adversarial attacks on learning algorithms and enhance the robustness of DRL agents, it is crucial to evaluate the performance of the agents under any potential adversarial attacks with certain constraints. In other words, identifying a universal and strong adversary is essential.

\begin{figure}[t!]
    \centering
    \includegraphics[width=0.37\textwidth]{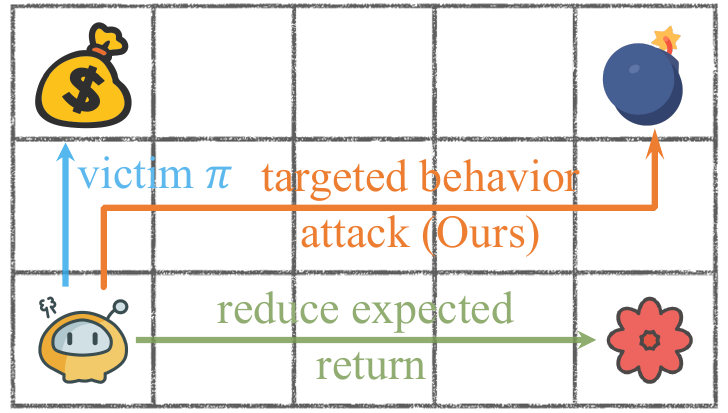}
    \caption{An example illustrating the distinction between our approach and generic attacks.}
    \label{fig:case}
\end{figure}

Existing methods pay little attention to devising universal, efficient, targeted behavior attacks. Firstly, several methods primarily focused on reducing the cumulative reward often lack specified attack targets. Prior research~\citep{zhang2020robust, zhang2021robust, sun2022who} considers training strong adversaries by perturbing state observations of victims to achieve the worst-case expected return. However, rewards lack the expressiveness to adequately encode complex safety requirements~\cite{vamplew2022scalar, 978-3-030-57628-8_1}. Additionally, requiring the victim's training rewards to craft such attacks is generally impractical. Therefore, only quantifying the decrease in cumulative reward can be too generic and result in suboptimal attack performance, particularly when adversaries are intended to execute specific safety-related attacks. Consider the scenario depicted in Figure~\ref{fig:case}, where a robot's objective is to collect coins. Previous attack methods aim at inducing the robot away from the coins by minimizing its expected return. However, this approach overlooks specific unsafe behaviors, such as manipulating the robot to collide with a bomb. Secondly, the previous targeted attack only considered predefined targets, which resulted in rigidity and inefficiency. ~\citep{abs-1905-12282, ijcai2017p525} mainly focuses on misleading the agent towards a predetermined state or target policy, overlooking specific behaviors. Additionally, the difficulty of providing a well-designed targeted policy makes these methods hard to apply. In a broader context, these adversarial attacks are incapable of controlling the behaviors of agents as a form of universal attack.

In this paper, we present a novel adversarial attack method, \ourmethod, which focuses on Adve\textbf{R}sarial \textbf{A}ttacks against deep reinforcement learning agents for \textbf{T}argeted behavior. \ourmethod consists of three core components: an intention policy, an adversary, and a weighting function, all trained simultaneously. Unlike previous methods that rely on predefined target policies, \ourmethod dynamically trains an intention policy that aligns with human preferences, providing a flexible and adaptive behavioral target for the adversary. By leveraging advances in preference-based reinforcement learning (PbRL)~\citep{lee2021pebble, park2022surf, liu2022metarewardnet, bai2024efficient}, the intention policy efficiently captures human intent during the training process. \ourmethod employs the adversary to perturb the victim agent’s observations, guiding the agent towards the behaviors specified by the intention policy. To further enhance attack effectiveness, we introduce a weighting function that adjusts the state occupancy measure, optimizing the distribution of states visited during training. This adjustment improves both the performance and efficiency of the attack. Through iterative refinement, \ourmethod steers the victim agent toward specific human-desired behaviors with greater precision than existing adversarial attack methods.

Our contributions are summarized as follows:
\textbf{(1)} We propose a universal targeted behavior attack method against DRL agents, designed to induce specific behaviors in a victim agent across a wide range of tasks.
\textbf{(2)} We provide a theoretical analysis of \ourmethod, offering a convergence guarantee under clearly defined conditions, which enhances the understanding of its effectiveness.
\textbf{(3)} Through extensive experiments across various domains, we demonstrate that \ourmethod significantly outperforms existing adversarial attack methods, showing that both online and offline RL agents, including Decision Transformer, are susceptible to our approach.
\textbf{(4)} We introduce two variants, RAT-ATLA and RAT-WocaR, which demonstrate how \ourmethod can be effectively employed to enhance the robustness of DRL agents through adversarial training, showing its versatility in both attack and defense.

\section{Related Work}
\noindent \textbf{Adversarial Attacks on State Observations in DRL.}
\citet{HuangPGDA17} applies the Fast Gradient Sign Method (FGSM)~\citep{GoodfellowSS14} to compute adversarial perturbations, directing the victim policy towards suboptimal actions. ~\citet{3237383.3238064} introduces a strategy to make the victim choose the worst action based on its Q-function. ~\citet{Gleave2020Adversarial} focuses on adversarial attacks within the context of a two-player Markov game rather than altering the agent's observation. ~\citet{zhang2020robust} proposes the state-adversarial MDP (SA-MDP) and develops two adversarial attack methods, Robust Sarsa (RS) and Maximal Action Difference (MAD). SA-RL~\citep{zhang2021robust} optimizes an adversary to perturb states using end-to-end RL. PA-AD~\citep{sun2022who} utilizes an RL-based "director" to determine the best policy perturbation direction and an optimization-based "actor" to generate perturbed states accordingly. Another line of work focuses on steering DRL agents toward specific states or policies. ~\citet{DBLP:conf/ijcai/LinHLSLS17, 9811574} propose targeted adversarial attack methods against DRL agents, aimed at directing the agent to a specific state. ~\citet{DBLP:journals/corr/abs-1905-12282} offer a novel approach by attacking the agent to mimic a target policy. However, these methods often require access to the victim's training reward or a predetermined target state or policy, which may be impractical. Our method differs from these methods by emphasizing the manipulation of the victim's behaviors without needing access to the victim's training reward or a pre-defined target state or policy.

\noindent \textbf{Robustness for State Observations in DRL.}
Training DRL agents with perturbed state observations from adversaries has been explored in various studies. \citet{pmlr-v119-shen20b, NEURIPS2021_dbb42293} focus on a strategy, ensuring that the policy produces similar outputs for similar inputs, which has demonstrated certifiable performance in video games. Another research direction, as presented in \citet{pmlr-v70-pinto17a, 8206245, 3237383.3238064}, aims to enhance an agent's robustness by training it under adversarial attacks. \citet{zhang2021robust} proposes ATLA, a method that alternates between training an RL agent and an RL adversary, significantly enhancing policy robustness. Building on this concept, ~\citet{sun2022who} proposed PA-ATLT, which employs a similar approach but utilizes a more advanced RL attacker. And several methods proposed by \citet{fischer2019online,pmlr-v100-lutjens20a}, concentrate on the lower bounds of the Q-function to certify an agent's robustness at every step. WocaR-RL~\citep{NEURIPS2022_8d6b1d77} is an efficient method that directly estimates and optimizes the worst-case reward of a policy under attacks without requiring extra samples for learning an attacker.

\noindent \textbf{Preference-based RL.} 
PbRL provides an effective way to incorporate human preferences into agent learning. ~\citet{christiano2017deep} proposes a foundational framework for PbRL. ~\citet{NEURIPS2018_8cbe9ce2} utilizes expert demonstrations to initialize the policy, besides learning the reward model from human preferences. Nonetheless, these earlier methods often require extensive human feedback, which is typically not feasible in practical scenarios. Recent studies have addressed this limitation: ~\citet{lee2021pebble} develops a feedback-efficient PbRL algorithm, leveraging unsupervised exploration and reward relabeling. ~\citet{park2022surf} furthers feedback efficiency through semi-supervised reward learning and data augmentation. Meanwhile, ~\citet{liang2021reward} proposes an intrinsic reward to enhance exploration. Continuing this trend, ~\citet{liu2022metarewardnet} improves feedback efficiency by aligning the Q-function with human preferences. Additionally, several works~\citep{bai2024efficient, liu2024pearl} have been dedicated to improving feedback efficiency by providing diverse insights. In our research, we employ PbRL to capture human intent and train an intention policy, which serves as the learning target for training adversaries.

\section{Problem Setup and Notations}
\noindent \textbf{The Victim Policy}. In RL, agent learning can be modeled as a finite-horizon Markov Decision Process (MDP) defined as a tuple $(\gS, \gA, \gR, \gP, \gamma)$. $\gS$ and $\gA$ denote state and action space, respectively. $\gR: \gS \times \gA \times \gS \rightarrow \mathbb{R}$ is the reward function, and $\gamma \in (0,1)$ is the discount factor. $\gP:\gS \times \gA \times \gS \rightarrow [0, 1]$ denotes the transition dynamics, which determines the probability of transferring to $\rvs^\prime$ given state $\rvs$ and action $\rva$. We denote the stationary policy $\pi_\nu: \gS \rightarrow \gP(\gA)$, where $\nu$ are parameters of the victim. We suppose the victim policy is fixed and uses the approximator.

\noindent \textbf{Threat Model.} To study targeted  behavior attack with human preferences, we formulate it as rewarded state-adversarial Markov Decision Process (\MDP). Formally, a \MDP is a tuple $(\gS, \gA, \gB, \widehat{\gR}, \gP, \gamma)$. 
The adversary $\pi_\alpha: \gS \rightarrow \gP(\gS)$ perturbs the states before the victim observes them, where $\alpha$ are parameters of the adversary. The adversary perturbs the state $\rvs$ into $\tilde{\rvs}$ restricted by $\gB(\rvs)$ (i.e., $\tilde{\rvs} \in \gB(\rvs)$). $\gB(\rvs)$ is defined as a small set $\{\tilde{\rvs} \in \gS: \parallel \rvs-\tilde{\rvs} \parallel_p \le \epsilon\}$, which limits the attack power of the adversary, and $\epsilon$ is the attack budget. 
Since directly generating $\tilde{\rvs} \in \gB(\rvs)$ is hard, the adversary learns to produce a Gaussian noise $\Delta$ with $\ell_\infty(\Delta)$ less than 1, and we obtain the perturbed state through $\tilde{\rvs} = \rvs + \Delta * \epsilon$.
The victim takes action according to the observed $\tilde{\rvs}$, while true states in the environment are not changed. Recall that $\pi_{\nu \circ \alpha}$ denotes the perturbed policy caused by adversary $\pi_\alpha$, i.e., $\pi_{\nu \circ \alpha}(\cdot|\rvs) =\pi_{\nu}\left( \cdot|\pi_{\alpha}(\rvs) \right), \forall \rvs \in \gS$.

Unlike SA-MDP~\citep{zhang2020robust}, \MDP introduces $\widehat{\gR}$, which learns from human preferences. The target of \MDP is to solve the optimal adversary $\pi_\alpha^*$, which enables the victim to achieve the maximum cumulative reward (i.e., from $\widehat{\gR}$) over all states. Lemma~\ref{lem:equivalence} shows that solving the optimal adversary in \MDP is equivalent to finding the optimal policy in MDP $\hat{\gM} = (\gS, \hat{\gA}, \widehat{\gR}, \widehat{\gP}, \gamma)$, where $\hat{\gA} = \gS$ and $\widehat{\gP}$ is the transition dynamics of the adversary.

\begin{lemma} 
\label{lem:equivalence}
Given a \MDP $\gM = (\gS,\gA, \gB, \widehat{\gR}, \gP, \gamma)$ and a fixed victim policy $\pi_\nu$, there exists a MDP $\hat{\gM} = (\gS, \hat{\gA}, \widehat{\gR}, \widehat{\gP}, \gamma)$ such that the optimal policy of $\hat{\gM}$ is equivalent to the optimal adversary $\pi_\alpha$ in \MDP given a fixed victim, where $\widehat{\gA}=\gS$ and
\begin{equation*}
    \widehat{\gP}(\rvs^\prime|\rvs,\rva) = \sum_{\rva \in \gA}{\pi_\nu(\rva|\widehat{\rva})} \gP(\rvs^\prime|\rvs,\rva) \quad \text{for} \ \rvs,\rvs^\prime \in \gS \ \textrm{and} \ \widehat{\rva} \in \widehat{\gA}.
\end{equation*}
\end{lemma}

\section{Method}\label{sec:method}
In this section, we introduce \ourmethod, a generic framework adaptable to any RL algorithm for conducting targeted behavior attack against DRL learners. \ourmethod is composed of three integral components: an intention policy $\pi_\theta$, the adversary $\pi_\alpha$, and the weighting function $h_\omega$, all of which are trained in tandem.
The fundamental concept behind \ourmethod is twofold:
\textbf{(1)} It develops an intention policy to serve as the learning objective for the adversary.
\textbf{(2)} A weighting function is trained to adjust the state occupancy measure of replay buffer, and the training of $\pi_\alpha$ and $h_\omega$ is formulated as a bi-level optimization problem. The framework of \ourmethod is depicted in Figure~\ref{fig:framework}, with a comprehensive procedure outlined in Appendix~\ref{appendix:pseudo_code}.

\begin{figure}[!h]
\centering
\includegraphics[width=0.44\textwidth, height=0.24\textwidth]{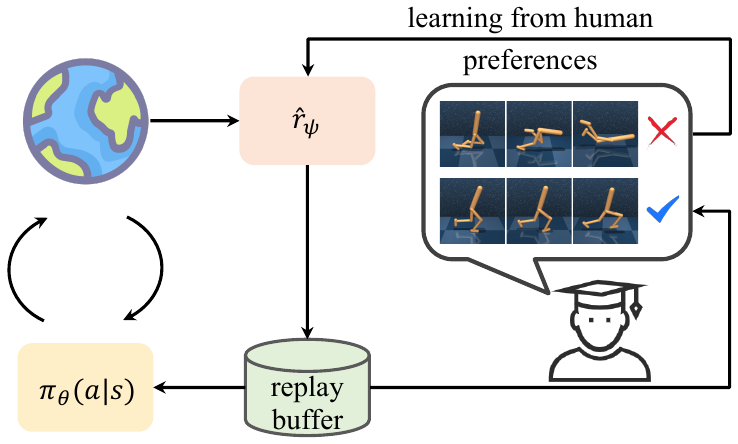}
\caption{\textbf{Diagram of PbRL.} The reward model $\widehat{r}_\psi$ is trained to align with human intention, providing estimations of rewards for policy learning. The policy is optimized by using transitions relabeled by the up-to-date reward model.}
\label{fig:pbrl}
\end{figure}

\begin{figure*}[!ht]
\centering
\includegraphics[width=0.98\linewidth]{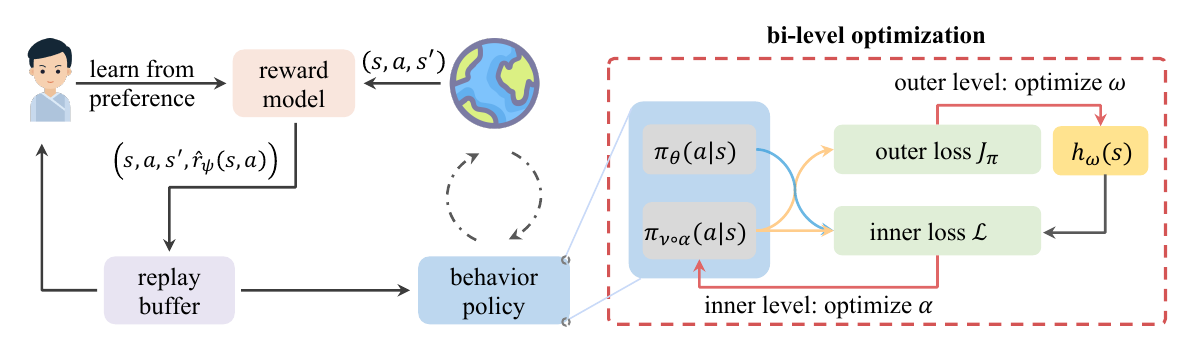}
\caption{\textbf{Overview of \ourmethod.} During training, it learns the intention policy $\pi_\theta$ and the reward model $\widehat{r}_\psi$, following the principles of PbRL. Simultaneously, it trains an adversary $\pi_\alpha$ and a weighting function $h_\omega$ within a bi-level optimization framework.
\textbf{In the inner-level}, the adversary is optimized such that the perturbed policy aligns with the intention policy. A validation loss $\outerloss$ is introduced, serving as a metric to assess the adversary's performance. \textbf{In the outer-level}, the weighting function is updated to improve the performance of the adversary by minimizing the outer loss $\outerloss$. 
}
\label{fig:framework}
\end{figure*}

\subsection{Learning Intention Policy}\label{subsec:intention_policy}
\ourmethod is designed to find an optimal adversary capable of manipulating the victim's behaviors in alignment with human intentions. To achieve this, we consider capturing human intentions and training an intention policy $\pi_\theta$, which translates these abstract intentions into action-level behaviors. A practical approach to realizing this concept is through PbRL, a method that aligns the intention policy with human intent without the need for reward engineering. As depicted in Figure~\ref{fig:pbrl}, within the PbRL framework, the agent does not rely on a ground-truth reward function. Instead, humans provide preference labels comparing two agent trajectories, and the reward model $\widehat{r}_\psi$ is trained to match human preferences~\citep{christiano2017deep}.

Formally, we denote a state-action sequence of length $k$, $\{\rvs_{t+1}, \rva_{t+1}, \cdots, \rvs_{t+k}, \rva_{t+k}\}$ as a segment $\sigma$. 
Given a pair of segments $(\sigma^0, \sigma^1)$, humans provide a preference label $y$ indicating which segment is preferred. Here, $y$ represents a distribution, specifically $y \in \{(0,1), (1,0), (0.5,0.5)\}$. In accordance with the Bradley-Terry model~\citep{pref-model-orig}, we construct a preference predictor as shown in~(\ref{eq:preference_r}):
\begin{equation}
    P_{\psi}[\sigma^0\succ\sigma^1] = \frac{\exp\sum_{t}\widehat{r}_{\psi}(\rvs_t^0, \rva_t^0)}{\sum_{i\in\{0,1\}}\exp\sum_{t}\widehat{r}_{\psi}(\rvs_t^i, \rva_t^i)},
\label{eq:preference_r}
\end{equation}
where $\sigma^0\succ\sigma^1$ indicates a preference for $\sigma^0$ over $\sigma^1$. This predictor determines the probability of a segment being preferred, proportional to its exponential return. 

The reward model is optimized to align the predicted preference labels with human preferences using a cross-entropy loss, as expressed in the following equation:
\begin{equation}
\begin{aligned}
   \mathcal{L}(\psi) = -\underset{(\sigma^0,\sigma^1,y)\sim \mathcal{D}}{\mathbb{E}}
   \Big[ \sum_{i=0}^{1} y(i)\log P_\psi[\sigma^i\succ\sigma^{1-i}] \Big],
\label{eq:reward_loss} 
\end{aligned}
\end{equation}
where $\mathcal{D}$ represents a dataset of triplets $(\sigma^0, \sigma^1, y)$ that consist of segment pairs and corresponding human preference labels. By minimizing the cross-entropy loss as defined in~(\ref{eq:reward_loss}), we derive an estimated reward function $\widehat{r}_\psi$. This function is then utilized to provide reward estimations for policy learning using any RL algorithm. Following PEBBLE~\citep{lee2021pebble}, we employ the Soft Actor-Critic (SAC)~\citep{haarnoja2018soft} algorithm to train the intention policy $\pi_\theta$. The Q-function $Q_\phi$ is optimized by reducing the Bellman residual, as defined below:
\begin{equation}
    J_Q(\phi) = \mathop{\mathbb{E}}\limits_{\tau_t \sim \gB} 
    \left[\left(Q_\phi(\rvs_t, \rva_t) - \widehat{r}_t - \gamma {\bar V}(\rvs_{t+1}) \right)^2 \right],
\label{eq:q_loss}
\end{equation}
where ${\bar V}(\rvs_t) = \mathbb{E}_{\rva_t\sim \pi_\theta} \big[ Q_{\bar \phi}(\rvs_t, \rva_t) - \mu \log \pi_{\theta}(\rva_t| \rvs_t) \big]$, $\tau_t = (\rvs_t, \rva_t, \widehat{r}_t, \rvs_{t+1})$ represents the transition at time $t$, with $\bar{\phi}$ being the parameter of the target soft Q-function.
The intention policy $\pi_\theta$ is updated to minimize the following loss:
\begin{equation}
    J_\pi(\theta) = \mathbb{E}_{\rvs_t\sim \mathcal{B}, \rva_t\sim \pi_\theta} \Big[ \mu \log \pi_\theta (\rva_t| \rvs_t) - Q_\phi (\rvs_t, \rva_t) \Big],
\label{eq:pi_loss}
\end{equation}
where $\mu$ is the temperature parameter. 

In this way, \ourmethod effectively captures human intent via the reward model $\widehat{r}_\psi$ and leverages $\pi_\theta$ to provide behavior-level guidance for the training of the adversary.

\subsection{Learning Adversary and Weighting Function}\label{subsec:adversary_policy}
To steer the victim policy towards behaviors desired by humans, \ourmethod trains the adversary by minimizing the Kullback-Leibler (KL) divergence between the perturbed policy $\perturbedpi$ and the intention policy $\pi_\theta$. Additionally, certain pivotal moments during adversary training can significantly influence the success rate of attacks. To ensure a stable training process and enhance the adversary's performance, a weighting function $h_\omega$ is introduced to re-weight the state occupancy measure of dataset.

Formally, our method is formulated as a bi-level optimization algorithm. It alternates between updating the adversary $\pi_\alpha$ and the weighting function $h_\omega$ through inner and outer optimization processes. In the inner level, the adversary's parameters $\alpha$ are optimized by minimizing the re-weighted KL divergence between $\perturbedpi$ and $\pi_\theta$, as specified in~(\ref{eq:kl_divergence}). At the outer level, the weighting function is developed to identify crucial states and improve the adversary's performance, as guided by a performance metric of the adversary. This metric is represented as a meta-level loss $\outerloss$, detailed in~(\ref{eq:out_loss}). The whole objective of \ourmethod is formulated as:
\begin{equation}
\begin{aligned}
    \min_\omega \quad & \outerloss(\alpha(\omega)), \\
    \text{s.t.} \quad & \alpha(\omega) = \arg\min_\alpha \innerloss(\alpha; \omega, \theta).
\end{aligned}
\label{eq:objective}
\end{equation}

\noindent \textbf{Inner Loop: Training Adversary $\pi_\alpha$.}
In inner-level optimization, with the given intention policy $\pi_\theta$ and the weighting function $h_\omega$, the goal is to identify the optimal adversary. This is achieved by minimizing the re-weighted KL divergence between $\perturbedpi$ and $\pi_\theta$, as shown in equation~(\ref{eq:kl_divergence}):
\begin{equation}
    \innerloss(\alpha;\omega,\theta) = \mathop{\mathbb{E}}\limits_{\rvs\sim\mathcal{B}}
    \Big[ h_\omega(\rvs) \KL\big( \perturbedpi(\cdot|\rvs) \parallel \pi_\theta(\cdot|\rvs) \big) \Big],
\label{eq:kl_divergence}
\end{equation}
where $h_\omega(\rvs)$ represents the importance weights determined by the weighting function $h_\omega$.

Intuitively, the adversary is optimized to ensure that the perturbed policy $\perturbedpi$ aligns behaviorally with the intention policy. Concurrently, $h_\omega$ allocates varying weights to states, reflecting their differing levels of importance. Through the synergistic effort of the intention policy and the weighting function, our method effectively trains an optimal adversary.

\noindent \textbf{Outer Loop: Training Weighting Function $h_\omega$.}
In outer-level optimization, the goal is to develop a precise weighting function that can identify significant moments and refine the state occupancy measure of the replay buffer to enhance adversary learning. As the intention policy is the target for the perturbed policy, it becomes simpler to establish a validation loss. This loss measures the perturbed policy's performance and simultaneously reflects the adversary's effectiveness. Consequently, the weighting function is trained to differentiate the importance of states by optimizing this validation loss. The perturbed policy $\perturbedpi$ is assessed using a policy loss in~(\ref{eq:out_loss}), adapted from the policy loss in~(\ref{eq:pi_loss}):
\begin{equation}
    \outerloss(\alpha(\omega)) =  \mathop{\mathbb{E}}\limits_{\rvs_t\sim\gB,\atop \rva_t\sim\pi_{\nu \circ \alpha(\omega)}}
    \big[ \mu \log\pi_{\nu \circ \alpha(\omega)}( \rva_t|\rvs_t) - Q_\phi(\rvs_t, \rva_t)\big],
\label{eq:out_loss}
\end{equation}
where $\alpha(\omega)$ denotes $\alpha$ implicitly depends on $\omega$. The optimization process involves calculating the implicit derivative of $\outerloss(\alpha(\omega))$ with respect to $\omega$ and finding the optimal $\omega^*$ through optimization.

\noindent \textbf{Practical Implementation.}
A one-step gradient update is used to approximate $\arg\min_\alpha$, as shown in~(\ref{eq:approximation}), thus establishing a connection between $\alpha$ and $\omega$:
\begin{equation}
    \hata(\omega) \approx \at - \innerlrt \left. \grada \innerloss(\alpha;\omega,\theta) \right|_{\at}.
\label{eq:approximation}
\end{equation}
The gradient of the outer loss with respect to $\omega$ is then determined using the chain rule:
\begin{equation}
    \begin{aligned}
        \left. \gradw \outerloss(\alpha(\omega)) \right|_{\wt}
        &= \left. \gradhata \outerloss(\hata(\omega)) \right|_{\hatat}
        \left. \gradw \hatat(\omega) \right|_{\wt} \\
        &= \sum_{\rvs} f(\rvs) \cdot \left. \gradw h(\rvs) \right|_{\wt},\\
    \end{aligned}
\end{equation}
where $f(\rvs) = -\innerlrt \cdot ( \gradhata \outerloss(\alpha(\omega)) )^\top \grada \KL( \pi_{\nu \circ \alpha}(\cdot|\rvs) \parallel \pi_\theta(\cdot|\rvs) )$. The detailed derivation is provided in Appendix~\ref{appendix:derivation}. The essence of this step is to establish and compute the relationship between $\alpha$ and $\omega$. By obtaining the implicit derivative, \ourmethod updates the parameters of the weighting function using gradient descent with an outer learning rate.

\subsection{Theoretical Analysis}
We provide convergence guarantee of \ourmethod. In Theorem~\ref{th:outer_loss_convergence_rate}, we demonstrate the convergence rate of the outer loss. We demonstrate that the gradient of the outer loss with respect to $\omega$ will converge to zero. Consequently, \ourmethod learns a more effective adversary by leveraging the importance of the weights generated by the optimal weighting function. Theorem~\ref{th:inner_loss_convergence} addresses the convergence of the inner loss. We prove that the inner loss of \ourmethod algorithm converges to critical points under certain reasonable conditions, thereby ensuring that the parameters of the adversary can converge towards the optimal parameters. Detailed theorems and their proofs are available in Appendix~\ref{appendix:proofs}.

\begin{theorem}
    Suppose $\outerloss$ is Lipschitz-smooth with constant L, the gradient of $\outerloss$ and $\innerloss$ is bounded by $\rho$. Let the training iterations be $T$, the inner-level optimization learning rate $\innerlrt=\min\{1, \frac{c_1}{T}\}$ for some constant $c_1 > 0$ where $\frac{c_1}{T} < 1$. Let the outer-level optimization learning rate $\outerlrt=\min\{\frac{1}{L}, \frac{c_2}{\sqrt{T}}\}$ for some constant $c_2 > 0$ where $c_2 \le \frac{\sqrt{T}}{L}$, and $\sum_{t=1}^\infty\outerlrt \le \infty, \sum_{t=1}^\infty\outerlrt^2 \le \infty$. The convergence rate of $\outerloss$ achieves
    \begin{equation}
        \min_{1 \le t \le T} \mathbb{E}\left[ \left\| \nabla_\omega \outerloss(\atnext(\wt)) \right\|^2 \right] \le \mathcal{O}\left(\frac{1}{\sqrt{T}}\right).
    \end{equation}
    \label{th:outer_loss_convergence_rate}
\end{theorem}

\begin{theorem}
    Suppose $\outerloss$ is Lipschitz-smooth with constant L, the gradient of $\outerloss$ and $\innerloss$ is bounded by $\rho$. Let the training iterations be $T$, the inner-level optimization learning rate $\innerlrt=\min\{1, \frac{c_1}{T}\}$ for some constant $c_1 > 0$ where $\frac{c_1}{T} < 1$. Let the outer-level optimization learning rate $\outerlrt=\min\{\frac{1}{L}, \frac{c_2}{\sqrt{T}}\}$ for some constant $c_2 > 0$ where $c_2 \le \frac{\sqrt{T}}{L}$, and $\sum_{t=1}^\infty\outerlrt \le \infty, \sum_{t=1}^\infty\outerlrt^2 \le \infty$. $\innerloss$ achieves
    \begin{equation}
        \lim_{t \to \infty} \E \left[ \left\| \grada \innerloss(\at; \wt) \right\|^2 \right] = 0.
    \end{equation}
    \label{th:inner_loss_convergence}
\end{theorem}

\section{Experiments}
In this section, we evaluate our method using a range of robotic simulation manipulation tasks from Meta-world~\citep{yu2020meta} and continuous locomotion tasks from MuJoCo~\citep{6386109}. Our objective is to address the following key questions:
\textbf{(1)} Does our method have the capacity to implement universal targeted  behavior attack against DRL learners?
\textbf{(2)} Can our approach successfully deceive a commonly used offline RL method, such as the Decision Transformer~\citep{chen2021decision}, to execute specific behaviors?
\textbf{(3)} Does our method contribute to enhancing an agent's \textbf{robustness} through adversarial training?
\textbf{(4)} Are the individual components within our approach \textbf{effective}?
The responses to problems $(1)-(4)$ are addressed in Sections~\ref{subsec:case_one} through \ref{subsec:ablation_study}, respectively. A detailed description of the experimental setup is available in Appendix~\ref{appendix:exp_detail}.

\begin{table*}[!ht]
\centering
\caption{The average attack success rate, along with the standard deviation, is calculated for various evasion attacks against victim agents in both scenarios. The results are averaged over 30 episodes. Full results are available at Appendix~\ref{appendix:full_experiment_results}.}
\label{tab:case_one_results}
\resizebox{0.99\textwidth}{!}{
\begin{tabular}{cl|ccccc|c}
\toprule
& \multicolumn{1}{c|}{Task} & PA-AD (oracle) & PA-AD & SA-RL (oracle) & SA-RL & Random & \ourmethod (ours) \\
\midrule
\parbox[t]{3mm}{\multirow{8}{*}{\pix\rotatebox[origin=c]{90}{\footnotesize{Manipulation}}}} & ~~
Door Lock & 4.50 \stdv{4.00} & 3.50 \stdv{6.63} & 76.50 \stdv{14.97} & 39.50 \stdv{30.48} & 0.00 \stdv{0.00} & 87.00 \stdv{10.00} \\ & ~~ 
Door Unlock & 0.00 \stdv{0.00} & 0.00 \stdv{0.00} & 11.11 \stdv{13.43} & 0.56 \stdv{0.00} & 0.00 \stdv{0.00} & 97.00 \stdv{6.63} \\ & ~~ 
Window Open & 0.00 \stdv{0.00} & 0.00 \stdv{0.00} & 30.00 \stdv{21.19} & 8.00 \stdv{15.13} & 0.00 \stdv{0.00} & 72.50 \stdv{19.62} \\ & ~~ 
Window Close & 0.00 \stdv{0.00} & 0.50 \stdv{0.00} & 99.00 \stdv{3.00} & 23.50 \stdv{37.22} & 0.00 \stdv{0.00} & 72.50 \stdv{40.01} \\ & ~~ 
Drawer Open & 0.00 \stdv{0.00} & 0.00 \stdv{0.00} & 100.00 \stdv{0.00} & 26.00 \stdv{27.28} & 0.00 \stdv{0.00} & 97.50 \stdv{4.00} \\ & ~~ 
Drawer Close & 0.00 \stdv{0.00} & 0.00 \stdv{0.00} & 57.50 \stdv{18.00} & 4.00 \stdv{8.00} & 0.00 \stdv{0.00} & 76.00 \stdv{24.98} \\ & ~~ 
Faucet Open & 1.50 \stdv{4.00} & 2.50 \stdv{4.00} & 63.50 \stdv{20.52} & 0.00 \stdv{0.00} & 0.00 \stdv{0.00} & 84.00 \stdv{21.19} \\ & ~~ 
Faucet Close & 0.00 \stdv{0.00} & 0.00 \stdv{0.00} & 66.50 \stdv{16.85} & 4.50 \stdv{9.22} & 0.00 \stdv{0.00} & 91.00 \stdv{6.71} \\
\midrule
\parbox[t]{2mm}{\multirow{8}{*}{\pix\rotatebox[origin=c]{90}{Opposite}}} & ~~ 
Door Lock & 9.50 \stdv{7.48} & 10.00 \stdv{9.17} & 8.00 \stdv{13.42} & 2.00 \stdv{0.00} & 1.00 \stdv{3.00} & 99.00 \stdv{3.00} \\ & ~~ 
Door Unlock & 3.00 \stdv{5.00} & 4.00 \stdv{4.58} & 8.00 \stdv{18.33} & 6.00 \stdv{12.00} & 0.00 \stdv{0.00} & 98.50 \stdv{4.00} \\ & ~~ 
Window Open & 15.50 \stdv{12.21} & 17.00 \stdv{11.14} & 15.00 \stdv{16.61} & 7.00 \stdv{16.12} & 1.00 \stdv{3.00} & 77.50 \stdv{33.41} \\ & ~~ 
Window Close & 38.50 \stdv{23.69} & 55.00 \stdv{14.70} & 63.00 \stdv{34.70} & 20.00 \stdv{39.80} & 5.50 \stdv{5.00} & 99.00 \stdv{0.00} \\ & ~~ 
Drawer Open & 1.50 \stdv{4.00} & 0.50 \stdv{3.00} & 1.11 \stdv{0.00} & 3.00 \stdv{0.00} & 0.00 \stdv{0.00} & 85.50 \stdv{29.34} \\ & ~~ 
Drawer Close & 88.50 \stdv{7.81} & 79.00 \stdv{18.44} & 81.00 \stdv{20.88} & 63.00 \stdv{32.50} & 0.00 \stdv{0.00} & 92.00 \stdv{17.32} \\ & ~~ 
Faucet Open & 6.50 \stdv{9.00} & 10.00 \stdv{13.75} & 6.00 \stdv{18.00} & 0.00 \stdv{0.00} & 0.00 \stdv{0.00} & 81.50 \stdv{29.68} \\ & ~~ 
Faucet Close & 19.00 \stdv{13.27} & 32.00 \stdv{11.00} & 7.00 \stdv{12.81} & 8.00 \stdv{16.00} & 0.50 \stdv{0.00} & 96.00 \stdv{12.81} \\
\bottomrule
\end{tabular}
} 
\end{table*}

\subsection{Setup}
\noindent \textbf{Compared Methods.}
We compare our algorithm with Random attack and two state-of-the-art evasion attack methods, including
(1) \emph{Random}: a basic baseline that samples random perturbed observations via a uniform distribution.
(2) \emph{SA-RL}~\citep{zhang2021robust}: learning an adversary in the form of end-to-end RL formulation.
(3) \emph{PA-AD}~\citep{sun2022who}: combining RL-based “director” and non-RL “actor” to find state perturbations.
(4) \emph{\ourmethod}:  our proposed method, which collaboratively learns adversarial policy and weighting function with the guidance of intention policy.

\begin{figure*}[!ht]
\centering
\begin{tabular}{ccc}
\hspace*{-0.7em} \subfloat[Cheetah-Run Backwards]{\includegraphics[width=0.47\linewidth]{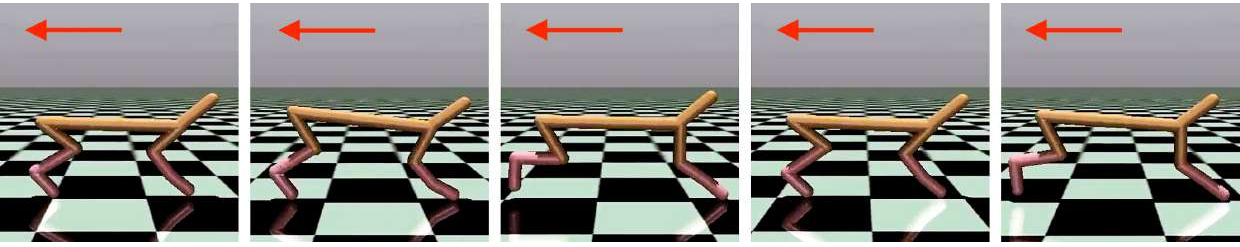}\label{fig:dt_a}}
& \hspace*{-1.1em} \subfloat[Walker-Stand on One Foot]{\includegraphics[width=0.47\linewidth]{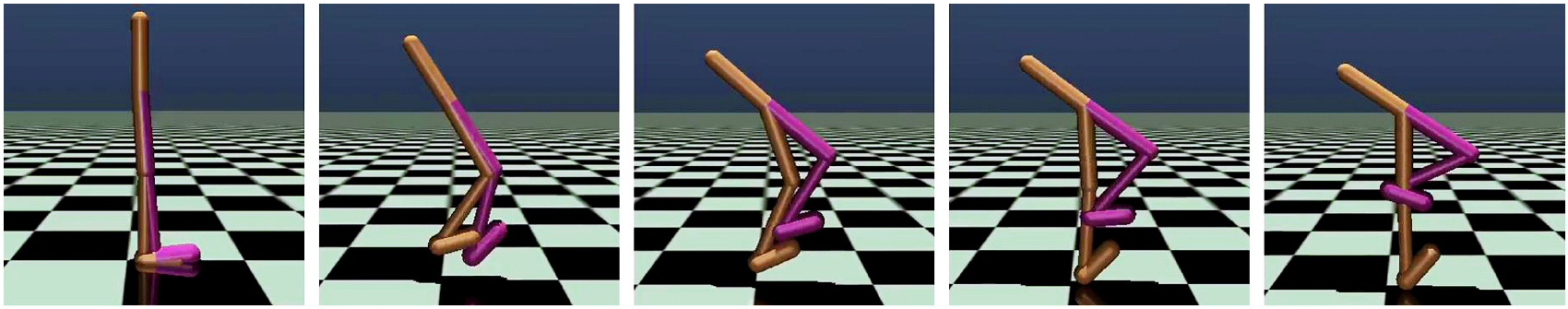}\label{fig:dt_b}}
\\
\hspace*{-0.7em} \subfloat[Cheetah-90 Degree Push-up]{\includegraphics[width=0.47\linewidth]{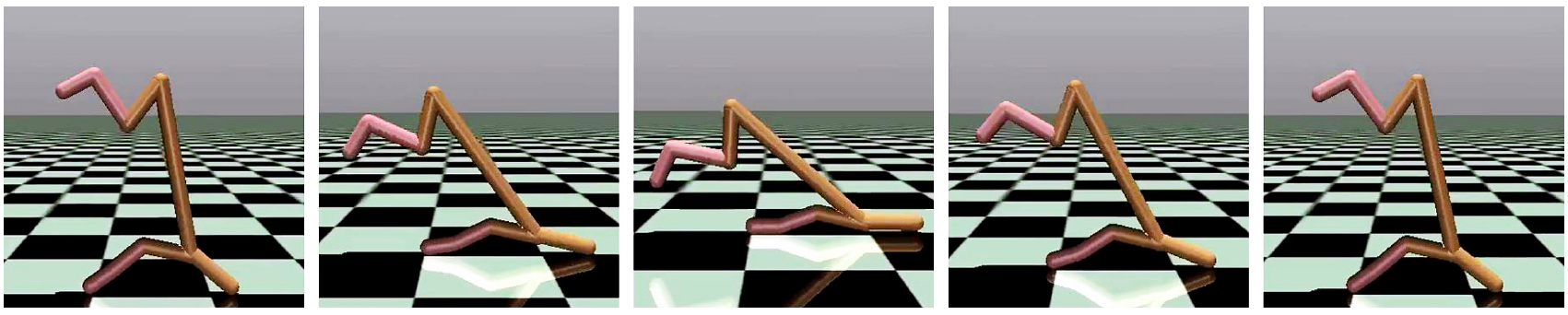}\label{fig:dt_c}}
& \hspace*{-1.1em} \subfloat[Walker-Dance]{\includegraphics[width=0.47\linewidth]{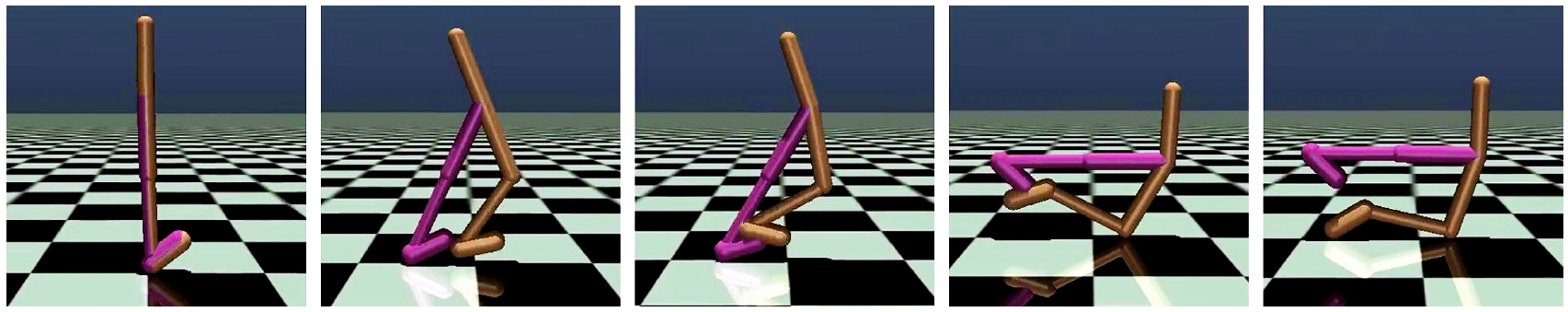}\label{fig:dt_d}}
\end{tabular}
\caption{Human desired behaviors behaved by the Decision Transformer under the attack of \ourmethod.}
\label{fig:fool_dt}
\end{figure*}

\noindent \textbf{Implementation Settings.}
In our experiments, all methods follow PEBBLE~\citep{lee2021pebble} to learn the reward model using the same number of preference labels. The key modification in employing PbRL is that the rewards in transitions are derived from the reward model $\widehat{r}_\psi$, rather than ground-truth rewards, and this model is trained by minimizing \eqref{eq:reward_loss}. Specifically, in the original versions of SA-RL~\citep{zhang2021robust} and PA-AD~\citep{sun2022who}, the negative value of the reward obtained by the victim is used to train adversaries. We adapt this by using estimated rewards from $\widehat{r}_\psi$. To evaluate performance effectively and expedite the training, we follow the foundational settings in PbRL~\citep{lee2021pebble, park2022surf, liu2022metarewardnet}, considering the use of a scripted teacher that always provides accurate preference labels. For the manipulation scenario, we employ 9000 labels across all tasks. In the opposite behavior scenario, the label usage varies: 1000 for Window Close, 3000 for Drawer Close, 5000 for Faucet Open, Faucet Close, and Window Open, and 7000 for Drawer Open, Door Lock, and Door Unlock. More information about the scripted teacher and preference collection is detailed in Appendix~\ref{appendix:pbrl_detail}. Moreover, to minimize the influence of PbRL, we include oracle versions of SA-RL and PA-AD, which utilize the ground-truth rewards of the targeted task. For implementing SA-RL\footnote{\url{https://github.com/huanzhang12/ATLA_robust_RL}} and PA-AD\footnote{\url{https://github.com/umd-huang-lab/paad_adv_rl}}, the official repositories are employed. As in most existing research~\citep{zhang2020robust, zhang2021robust, sun2022who}, we also use state attacks with $\normmax$ norm in our experiments.

To ensure a fair comparison, identical experimental settings (including hyper-parameters and neural networks) for reward learning are applied across all methods. We conduct a quantitative evaluation of all methods by comparing their average attack success rates. Comprehensive details on hyper-parameter settings, implementation details, and scenario designs are available in Appendix~\ref{appendix:exp_detail}.

\noindent \textbf{Evaluation Metrics.}\label{par:evaluation_metrics}
The success metric for adversarial attacks revolves around the proximity between the task-relevant object and its final goal position, denoted as $\I_{\|o - t\|_2 < \epsilon}$, where $\epsilon$ is a minimal distance threshold. In the manipulation scenario, we set $\epsilon=0.05$ (5cm). For the opposite behaviors scenario, we apply the success metrics and thresholds specified for each task by Meta-world~\citep{yu2020meta}. We summarize the all success metrics in our experiments in the Table~\ref{tab:task_metrics} in the Appendix~\ref{appendix:scenario_design}.

\subsection{Case I: Manipulation on DRL Agents}\label{subsec:case_one}
We first conduct an evaluation of our method and various other adversarial attack algorithms across two different scenarios, applying them to a range of simulated robotic manipulation tasks. Each victim agent is a well-trained SAC~\citep{haarnoja2018soft} agent, specialized for a specific manipulation task and trained for ${10}^6$ timesteps using the open-source code \footnote{\url{https://github.com/denisyarats/pytorch_sac}} available. Details on hyperparameter settings are provided in Appendix~\ref{appendix:victim_agents_settings}.

\noindent \textbf{Scenarios on Manipulation.}
In this scenario, our objective was to manipulate the victim (robotic arm) to grasp objects at locations distant from the originally intended target, rather than completing its initial task. Table~\ref{tab:case_one_results} presents the average attack success rates of both baseline methods and our approach across four manipulation tasks. The results indicate that the performance of \ourmethod significantly exceeds that of the baselines by a large margin. To reduce the influence of PbRL and further highlight the advantages of \ourmethod, we also trained baseline methods using the ground-truth reward function, labeling these as “oracle” versions. Notably, the performance of SA-RL (oracle) shows considerable improvement on several tasks compared to its preference-based counterpart. Nonetheless, \ourmethod still outperformed SA-RL with oracle rewards in most scenarios. These findings underscore the ability of \ourmethod to enable agents to effectively learn adversary based on human preferences. Additionally, it was observed that PA-AD struggles to perform effectively in manipulation tasks, even when trained with ground-truth rewards.

\noindent \textbf{Scenarios on Opposite Behaviors.}
Robotic manipulation holds significant practical value in real-world applications. Therefore, we craft this scenario to quantitatively assess the vulnerability of agents proficient in various manipulation skills. In this setup, each victim agent is expected to perform the opposite of its mastered task when subjected to the manipulator's targeted attack. For instance, a victim trained to open windows would be manipulated to close them instead. As demonstrated in Table~\ref{tab:case_one_results}, \ourmethod exhibits exceptional performance, consistently demonstrating clear advantages over the baseline methods across all tasks. This outcome reaffirms that \ourmethod is not only effective across a broad spectrum of tasks but also capable of efficiently learning adversaries aligned with human preferences. 

We observe that SA-RL and PA-AD exhibit relatively low attack success rates across numerous tasks, which can be attributed to the issue of distribution drift. This drift arises due to discrepancies between the data distribution sampled by the perturbed policy and the distribution corresponding to human-desired behaviors, leading to suboptimal performance.

\subsection{Case II: Manipulation on Sequence Model Agents}\label{subsec:case_two}
In this experiment, we show the vulnerability of offline RL agents and the capability of \ourmethod to fool them into acting human desired behaviors. As for the implementation, we choose some online models~\footnote{\url{https://huggingface.co/edbeeching}} as victims, which are well-trained by official implementation with D4RL. We choose two tasks, Cheetah and Walker, using expert-level Decision Transformer agents as the victims. As illustrated in Figure \ref{fig:fool_dt}, Decision Transformer reveals weaknesses that can be exploited, leading it to execute human-preferred behaviors rather than its intended tasks. Under adversarial manipulation, the Cheetah agent is shown to run backwards rapidly in Figure \ref{fig:dt_a} and perform a 90-degree push-up in Figure \ref{fig:dt_c}. Meanwhile, the Walker agent maintains superior balance on one foot in Figure \ref{fig:dt_b} and appears to dance with one leg raised in Figure \ref{fig:dt_d}. These outcomes indicate that \ourmethod is effective in manipulating these victim agents towards behaviors consistent with human preferences, highlighting the significant vulnerability of embodied agents to strong adversaries. This experiment is expected to spur further research into improving the robustness of offline RL agents and embodied AI systems.

\subsection{Robust Agents Training and Evaluation}\label{subsec:robust_agent}
A practical application of \ourmethod is in assessing the robustness of established models or in enhancing an agent's robustness via adversarial training. ATLA-PPO~\citep{zhang2021robust} presents a generic training framework aimed at improving robustness, which involves alternating training between an agent and an SA-RL attacker. PA-ATLA~\cite{sun2022who} follows a similar approach but employs a more advanced RL attacker, PA-AD. Drawing inspiration from previous works~\cite{zhang2021robust, NEURIPS2022_8d6b1d77}, we introduce two novel robust training methods: RAT-ATLA and RAT-WocaR. RAT-ATLA's central strategy is to alternately train an agent and a \ourmethod attacker, whereas RAT-WocaR focuses on directly estimating and minimizing the reward of the intention policy, obviating the need for extra samples to learn an attacker. Table~\ref{tab:robust2_single} compares the effectiveness of RAT-ATLA and RAT-WocaR for SAC agents on robotic simulation manipulation tasks against leading robust training methods. The experimental findings highlight two key points: first, RAT-ATLA and RAT-WocaR substantially improve agent robustness; and second, \ourmethod is capable of executing stronger attacks on robust agents, showcasing its effectiveness in challenging environments.

\begin{table}[!ht]
\centering
\caption{Average episode rewards $\pm $ standard deviation of robust agents under different attack methods, and results are averaged across 100 episodes.}
\resizebox{1.01\linewidth}{!}{
\begin{tabular}{ccccccc}
    \toprule
    \textbf{Task} & \textbf{Model} & \textbf{\ourmethod} & \textbf{PA-AD} & \textbf{SA-RL} & \textbf{Avg R}  \\
    \midrule
    \parbox[t]{4mm}{\multirow{4}{*}{\pix\rotatebox[origin=c]{90}{Door Lock}}}
    & RAT-ATLA  & 874\stdv{444}  & 628\stdv{486}  & 503\stdv{120}  & \textbf{668} \\
    & RAT-WocaR & 774\stdv{241}  & 527\stdv{512}  & 520\stdv{236}  & 607 \\
    & PA-ATLA  & 491\stdv{133}  & 483\stdv{15 }  & 517\stdv{129}  & 497 \\
    & ATLA-PPO & 469\stdv{11 }  & 629\stdv{455}  & 583\stdv{173}  & 545 \\
    \midrule
    \parbox[t]{4mm}{\multirow{4}{*}{\pix\rotatebox[origin=c]{90}{\footnotesize{Door Unlock}}}}
    & RAT-ATLA  & 477\stdv{203} & 745\stdv{75}  & 623\stdv{60}  & \textbf{615} \\
    & RAT-WocaR & 525\stdv{78}  & 647\stdv{502} & 506\stdv{39}  & 559 \\
    & PA-ATLA  & 398\stdv{12}  & 381\stdv{11}  & 398\stdv{79}  & 389 \\
    & ATLA-PPO & 393\stdv{36}  & 377\stdv{8}   & 385\stdv{26}  & 385 \\
    \midrule
    \parbox[t]{4mm}{\multirow{4}{*}{\pix\rotatebox[origin=c]{90}{\footnotesize{Faucet Open}}}}
    & RAT-ATLA  & 442\stdv{167} & 451\stdv{96}    & 504\stdv{55}   & 465 \\
    & RAT-WocaR & 1223\stdv{102}& 1824\stdv{413}  & 1575\stdv{389} & \textbf{1541}\\
    & PA-ATLA  & 438\stdv{53}  & 588\stdv{222}   & 373\stdv{32}   & 466 \\
    & ATLA-PPO & 610\stdv{293} & 523\stdv{137}   & 495\stdv{305}  & 522 \\
    \midrule
    \parbox[t]{4mm}{\multirow{4}{*}{\pix\rotatebox[origin=c]{90}{\footnotesize{Faucet Close}}}}
    & RAT-ATLA  & 1048\stdv{343} & 1223\stdv{348} & 570\stdv{453}  & 947 \\
    & RAT-WocaR & 1369\stdv{158} & 1416\stdv{208} & 3372\stdv{1311}& \textbf{2052}\\
    & PA-ATLA  & 661\stdv{279}  & 371\stdv{65}   & 704\stdv{239}  & 538 \\
    & ATLA-PPO & 1362\stdv{149} & 688\stdv{196}  & 426\stdv{120}  & 825 \\
    \bottomrule
\end{tabular}
}
\label{tab:robust2_single}
\end{table}

\subsection{Ablation Studies}\label{subsec:ablation_study}
\noindent \textbf{Contribution of Each Component.}
In our further experiments, we investigate the effect of each component in \ourmethod on Drawer Open and Drawer Close for the manipulation scenario and on Faucet Open, Faucet Close for the opposite behavior scenario. \ourmethod incorporates three essential components or techniques: the intention policy $\pi_\theta$, the weighting function $h_\omega$ and the combined behavior policy. As detailed in Table~\ref{tab:abla_component}, $\pi_\theta$ emerges as a pivotal component in \ourmethod, significantly boosting the attack success rate. This enhancement is largely due to its capability to mitigate distribution drift between the victim's behavior and the desired behavior. 

\begin{figure}[!ht]
\centering
\begin{tabular}{cc}
\hspace*{-1.0em} \subfloat[t-SNE Visualization]{\includegraphics[width=0.49\linewidth]{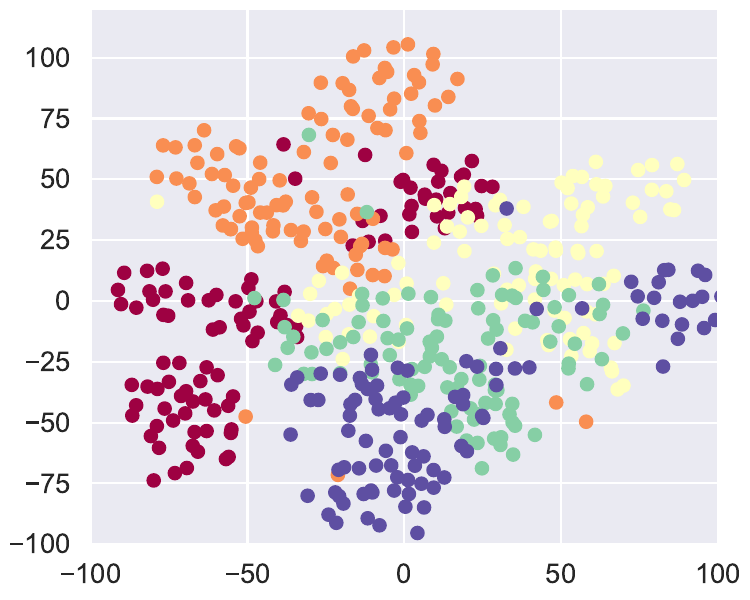}\label{fig:weight_tsne}}
& \hspace*{-0.7em} \subfloat[Weight Visualization]{\includegraphics[width=0.49\linewidth]{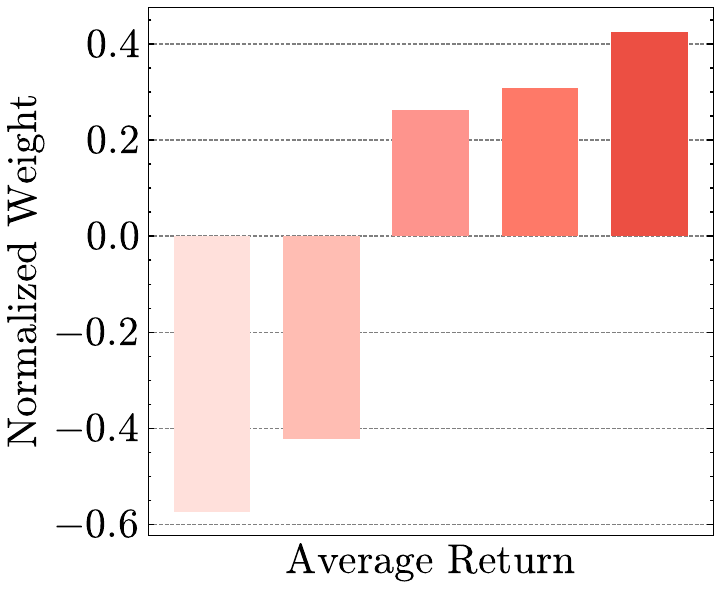}\label{fig:weight_bar}}
\end{tabular}
\caption{\textbf{Effects of the Weighting Function.} (a) Trajectory weights generated by the weighting function from various policies are visualized with t-SNE. (b) A visualization of the weights of trajectories of different qualities by five different policies.  
}
\label{fig:abla_weight}
\end{figure}

\begin{table*}[t]
\centering
\caption{Effects of each component in \ourmethod is evaluated based on the average attack success rate on four simulated robotic manipulation tasks. These results represent the mean success rate across five runs.}
\begin{tabular}{cccccc}
    \toprule
    \textbf{Scenario} & \textbf{Task} & \textbf{\ourmethod} & \textbf{\ourmethod w/o $h_\omega$} & \textbf{\ourmethod w/o $\pi_\theta$} & \textbf{\ourmethod w/o combined policy}  \\
    \midrule
    \multirow{2}{*}{Manipulation} 
    & Drawer Open & \textbf{99.1\%} & 91.3\% & 21.7\% & 68.0\% \\
    & Drawer Close & \textbf{80.9\%} & 70.2\% & 8.0\% & 26.0\% \\
    \midrule
    \multirow{2}{*}{Opposite} 
    & Faucet Open   & 84.4\% & \textbf{89.8}\% & 0.0\%  & 57.0\% \\
    & Faucet Close    & \textbf{95.1\%} & 94.1\% & 13.0\%  & 59.1\% \\
    \bottomrule
\end{tabular}
\label{tab:abla_component}
\end{table*}

\noindent \textbf{Effects of the Weighting Function.}
To further understand the weighting function proposed in Section~\ref{subsec:adversary_policy}, we conduct comprehensive experimental data analysis and visualization from multiple perspectives. We sample five perturbed policies uniformly, each representing a progressive stage of performance improvement before the convergence of \ourmethod. For each of these policies, 100 trajectories were rolled out, and their corresponding trajectory weight vectors were obtained via the weighting function. Utilizing t-SNE~\citep{vandermaaten08a} for visualization, Figure~\ref{fig:weight_tsne} showcases the weight vectors of different policies. This illustration reveals distinct boundaries between the trajectory weights of various policies, indicating the weighting function's ability to differentiate trajectories based on their quality. In Figure~\ref{fig:weight_bar}, trajectories with higher success rates in manipulation are represented in darker colors. The visualization suggests that the weighting function assigns higher weights to more successful trajectories, thereby facilitating the improvement of the adversary's performance. 

To thoroughly assess the impact of feedback amounts and attack budgets on the performance of \ourmethod, as well as the quality of the learned reward functions, we conducted extensive experiments. The detailed analyses and discussions of these aspects are provided in Appendix~\ref{appendix:extensive_exp}.

\section{Conclusion}
In this paper, we propose \ourmethod, a targeted behavior attack approach against DRL learners, which manipulates the victim to perform human-desired behaviors. \ourmethod involves an adversary adding imperceptible perturbations on the observations of the victim, an intention policy learned through PbRL as a flexible behavior target, and a weighting function to identify essential states for the efficient adversarial attack. We analyze the convergence of \ourmethod and prove that \ourmethod converges to critical points under some mild conditions. Empirically, we design two scenarios on several manipulation tasks in Meta-world, and the results demonstrate that \ourmethod outperforms the baselines in the targeted adversarial setting. 
Additionally, \ourmethod can enhance the robustness of agents via adversarial training. We further show embodied agents' vulnerability by attacking Decision Transformer on some MuJoCo tasks.

\bibliography{aaai25}
\newpage
\appendix
\onecolumn
\setcounter{theorem}{0}

\section{Full Procedure of \ourmethod}\label{appendix:pseudo_code}
In this section, we provide a detailed explanation of our method, including the pseudo code, the design of the Combined Behavior Policy, and the basic setup for preference-based reinforcement learning (PbRL). The procedures for our method are fully outlined in Algorithm~\ref{alg:RAT}, which describes the reward learning process and adversary updates in \ourmethod.
\begin{algorithm}[]
    \caption{\ourmethod}
    \label{alg:RAT}
    \begin{algorithmic}[1]
    \REQUIRE a fixed victim policy $\pi_\nu$, frequency of human feedback $K$, outer loss updating frequency $M$, task horizon $H$
    \STATE Initialize parameters of $Q_\phi$, $\pi_\theta$, $\widehat{r}_\psi$, $\pi_\alpha$ and $h_\omega$ 
    \STATE Initialize $\mathcal{B}$ and $\pi_\theta$ with unsupervised exploration
    \STATE Initialize preference data set $\mathcal{D} \leftarrow \emptyset$
    \FOR{each iteration}  
        \STATE \small{\textcolor[HTML]{ef5767}{\texttt{// Construct the combined behavior policy $\pi$}}} \\
        \IF{episode is done}
            \STATE $h\sim U(0,H)$
            \STATE $\pi^{1:h} = \pi_{\nu \circ \alpha}^{1:h}$ and $\pi^{h+1:H}=\pi_{\theta}^{h+1:H}$
        \ENDIF
        \STATE Take action $a_t \sim \pi$ and collect $s_{t+1}$
        \STATE Store transition into dataset $\mathcal{B} \leftarrow \mathcal{B} \cup \{ (s_t,a_t,\widehat{r}_\psi(s_t),s_{t+1}) \}$
        \STATE \small{\textcolor[HTML]{ef5767}{\texttt{// Query preference and Reward learning}}}
        \IF{iteration \% $K == 0$}
            \STATE Sample pair of trajectories $(\sigma^0,\sigma^1)$
            \STATE Query preference $y$ from manipulator
            \STATE Store preference data into dataset $\mathcal{D} \leftarrow \mathcal{D} \cup \{ (\sigma^0,\sigma^1,y) \}$
            \STATE Sample batch $\{(\sigma^0,\sigma^1,y)_i\}^n_{i=1}$ from $\mathcal{D}$
            \STATE Optimize~(\ref{eq:reward_loss}) to update $\widehat{r}_\psi$
        \ENDIF
        \STATE \small{\textcolor[HTML]{ef5767}{\texttt{// Inner loss optimization}}}
        \FOR{each gradient step}
            \STATE Sample random mini-batch transitions from $\mathcal{B}$
            \STATE Optimize $\pi_\alpha$: minimize~(\ref{eq:kl_divergence}) with respect to $\alpha$ 
        \ENDFOR
        \STATE \small{\textcolor[HTML]{ef5767}{\texttt{// Outer loss optimization}}}
        \IF{iteration \% $M == 0$}
            \STATE Sample random mini-batch transitions from $\mathcal{B}$
            \STATE Optimize $h_\omega$: minimize~(\ref{eq:out_loss}) with respect to $\omega$ 
        \ENDIF
        \STATE \small{\textcolor[HTML]{ef5767}{\texttt{// Intention policy learning}}}
        \STATE Update $Q_\phi$ and $\pi_\theta$ according to~(\ref{eq:q_loss}) and~(\ref{eq:pi_loss}), respectively.
    \ENDFOR
    \ENSURE adversary $\pi_\alpha$
\end{algorithmic} 
\end{algorithm}

\subsection{The Combined Behavior Policy} 
To address the inefficiencies caused by the distribution discrepancy between the learned policy $\pi_\theta$ and the perturbed policy $\pi_{\nu \circ \alpha}$, we developed a behavior policy $\pi$ for data collection inspired by Branched Rollout~\citep{NEURIPS2019_5faf461e}. Our approach combines the intention policy $\pi_\theta$ with the perturbed policy $\perturbedpi$ to balance exploration and exploitation during data collection. Specifically, we define the behavior policy $\pi$ as a combination of $\perturbedpi$ and $\pi_\theta$, where $\pi^{1:h} = \perturbedpi^{1:h}$ and $\pi^{h+1:H} = \pi_\theta^{h+1:H}$. Here, $h$ is sampled from a uniform distribution $U(0,H)$, where $H$ represents the task horizon. This combined policy is used to collect data, which is then stored in the replay buffer for training. By varying the point $h$ at which the policy switches from $\perturbedpi$ to $\pi_\theta$, we maintain a balance between exploration of new behaviors and reinforcement of learned behaviors, effectively mitigating the distribution discrepancy.

\subsection{Details of PbRL}\label{appendix:pbrl_detail}
In this section, we present details of the scripted teacher and the preference collection process, both of which are crucial components of PbRL. All methods in our paper follow the reward learning settings outlined in \citet{lee2021pebble}.

\noindent \textbf{Scripted Teacher.} To systematically evaluate the performance of our methods, we utilize a scripted teacher that provides preferences between pairs of trajectory segments based on the oracle reward function for online settings, like prior methods \citep{lee2021pebble, park2022surf, liu2022metarewardnet}. While leveraging human preference labels would be ideal, it is often impractical for quick and quantitative evaluations. The scripted teacher approximates human intentions by mapping states $\rvs$ and actions $\rva$ to ground truth rewards, providing immediate feedback. This function is designed to simulate the decision-making process of a human teacher by approximating their preferences based on cumulative rewards.

\noindent \textbf{Preference Collection.} During training, we query the scripted teacher for preference labels at regular intervals. A batch of segment pairs is sampled, and the cumulative rewards for each segment are calculated based on the rewards provided by the scripted teacher. The segment with the higher cumulative reward is assigned a label of 1, while the other is labeled 0. The computational cost of this process is proportional to the number of preference labels $M$ and the segment length $N$, resulting in a time complexity of $\gO(MN)$. However, this cost is negligible compared to adversary training, which involves more computationally expensive gradient calculations.

\section{Derivation of the Gradient of the Outer-level Loss} \label{appendix:derivation}
In this section, we present detailed derivation of the gradient of the outer loss $\outerloss$ with respect to the parameters of the weighting function $\omega$. According to the chain rule, we can derive that
\begin{equation}
\begin{aligned}
     & \left. \gradw \outerloss(\hata(\omega)) \right|_{\wt} \\
    =& \frac{\partial \outerloss(\hata(\omega))}{\partial \hata(\omega)}\Big|_{\hatat} \frac{\partial \hatat(\omega)}{\partial \omega} \Big|_{\wt} \\
    =& \frac{\partial \outerloss(\hata(\omega))}{\partial \hata(\omega)}\Big|_{\hatat} 
    \frac{\partial \hatat(\omega)}{\partial h(\rvs;\omega)} \Big|_{\wt} 
    \frac{\partial h(\rvs;\omega)}{\partial \omega} \Big|_{\wt} \\
    =& -\innerlrt \frac{\partial \outerloss(\hata(\omega))}{\partial \hata(\omega)}\Big|_{\hatat}
    \sum_{\rvs\sim \gB}{\frac{\partial \KL\left( \pi_{\nu \circ \alpha}(\rvs) \parallel \pi_\theta(\rvs) \right)}{\partial \alpha}} \Big|_{\at}
    \frac{\partial h(\rvs;\omega)}{\partial \omega} \Big|_{\wt} \\
    =& -\innerlrt \sum_{\rvs\sim \gB}
    \left(
    \frac{\partial \outerloss(\hata(\omega))}{\partial \hata(\omega)}\Big|_{\hatat}^\top {\frac{\partial \KL\left( \pi_{\nu \circ \alpha}(\rvs) \parallel \pi_\theta(\rvs) \right)}{\partial \alpha}} \Big|_{\at}
    \right)
    \frac{\partial h(\rvs;\omega)}{\partial \omega} \Big|_{\wt}.
\end{aligned}
\end{equation}
For brevity of expression, we let:
\begin{equation}
    f(\rvs) = \frac{\partial \outerloss(\hata(\omega))}{\partial \hata(\omega)}\Big|_{\hatat}^\top {\frac{\partial \KL\left( \pi_{\nu \circ \alpha}(\rvs) \parallel \pi_\theta(\rvs) \right)}{\partial \hata}} \Big|_{\at}.
\end{equation}
The gradient of outer-level optimization loss with respect to parameters $\omega$ is:
\begin{equation}
    \left. \gradw \outerloss(\hata(\omega))\right|_{\omega_t} =
    -\innerlrt \sum_{\rvs\sim \gB}{f(\rvs)\cdot \frac{\partial h(\rvs;\omega)}{\partial \omega} \Big|_{\omega_t}}.
\end{equation}

\section{Connection between \MDP and MDP}
\begin{lemma} \label{lem:equivalence}
    Given a \MDP $\gM = (\gS,\gA, \gB, \widehat{\gR}, \gP, \gamma)$ and a fixed victim policy $\pi_\nu$, there exists a MDP $\hat{\gM} = (\gS, \hat{\gA}, \widehat{\gR}, \widehat{\gP}, \gamma)$ such that the optimal policy of $\hat{\gM}$ is equivalent to the optimal adversary $\pi_\alpha$ in \MDP given a fixed victim, where $\widehat{\gA}=\gS$ and
    \begin{equation*}
        \widehat{\gP}(\rvs^\prime|\rvs,\rva) = \sum_{\rva \in \gA}{\pi_\nu(\rva|\widehat{\rva})} \gP(\rvs^\prime|\rvs,\rva) \quad \text{for} \ \rvs,\rvs^\prime \in \gS \ \text{and} \ \widehat{\rva} \in \widehat{\gA}.
    \end{equation*}
\end{lemma}

\section{Theoretical Analysis and Proofs}\label{appendix:proofs}
\subsection{Theorem 1: Convergence Rate of the Outer Loss}
\begin{lemma} \label{lem:lipschitz_upper_bound} (Lemma 1.2.3 in~\citet{nesterov1998introductory})
    If function $f(x)$ is Lipschitz smooth on $\mathbb{R}^n$ with constant $L$, then $\forall x, y \in \mathbb{R}^n$, we have
    \begin{equation}
        \left| f(y) - f(x) - f'(x)^\top (y - x) \right| \le \frac{L}{2} \left\| y - x \right\|^2.
    \end{equation}
\end{lemma}

\begin{proof}
    $\forall x, y \in \mathbb{R}^n$, we have
    \begin{equation}
    \begin{aligned}
        f(y) &= f(x) + \int_0^1 f'(x + \tau (y - x))^\top (y - x) d\tau \\
             &= f(x) + f'(x)^\top (y - x) + \int_0^1 [f'(x + \tau (y - x)) - f'(x)]^\top (y - x) d\tau.
    \end{aligned}
    \end{equation}
    Then we can derive that
    \begin{equation}
    \begin{aligned}
        \left| f(y) - f(x) - f'(x)^\top (y - x) \right|
        &= \left| \int_0^1 [f'(x + \tau (y - x)) - f'(x)]^\top (y - x) d\tau \right| \\
        & \le \int_0^1 \bigg| [f'(x + \tau (y - x)) - f'(x)]^\top (y - x) \bigg| d\tau \\
        & \le \int_0^1 \left\| f'(x + \tau(y - x)) - f'(x) \right\| \cdot \left\| y - x \right\| d\tau \\
        & \le \int_0^1 \tau L \left\| y - x \right\|^2 d\tau = \frac{L}{2} \left\| y - x \right\|^2,
    \end{aligned}
    \end{equation}
    where the first inequality holds for $\left| \int_a^b f(x) dx \right| \le \int_a^b \left| f(x) \right| dx$, the second inequality holds for Cauchy-Schwarz inequality, and the last inequality holds for the definition of Lipschitz smoothness.
\end{proof}

\begin{theorem}
    Suppose $\outerloss$ is Lipschitz-smooth with constant L, the gradient of $\outerloss$ and $\innerloss$ is bounded by $\rho$. Let the training iterations be $T$, the inner-level optimization learning rate $\innerlrt=\min\{1, \frac{c_1}{T}\}$ for some constant $c_1 > 0$ where $\frac{c_1}{T} < 1$. Let the outer-level optimization learning rate $\outerlrt=\min\{\frac{1}{L}, \frac{c_2}{\sqrt{T}}\}$ for some constant $c_2 > 0$ where $c_2 \le \frac{\sqrt{T}}{L}$, and $\sum_{t=1}^\infty\outerlrt \le \infty, \sum_{t=1}^\infty\outerlrt^2 \le \infty$. The convergence rate of $\outerloss$ achieves
    \begin{equation}
        \min_{1 \le t \le T} \mathbb{E}\left[ \left\| \nabla_\omega \outerloss(\atnext(\wt)) \right\|^2 \right] \le \mathcal{O}\left(\frac{1}{\sqrt{T}}\right).
    \end{equation}
\end{theorem}

\begin{proof}
    First,
    \begin{equation}
    \begin{aligned}
          & \outerloss(\hatatnextnext(\wtnext)) - \outerloss(\hatatnext(\wt)) \\
        = & \left\{ \outerloss(\hatatnextnext(\wtnext)) - \outerloss(\hatatnext(\wtnext)) \right\} + \left\{ \outerloss(\hatatnext(\wtnext)) - \outerloss(\hatatnext(\wt)) \right\}.
    \end{aligned}
    \label{eq:th1-1}
    \end{equation}
    Then we separately derive the two terms of~(\ref{eq:th1-1}). For the first term,
    \begin{equation}
    \begin{aligned}
            & \outerloss(\hatatnextnext(\wtnext)) - \outerloss(\hatatnext(\wtnext)) \\
        \le & \gradhata \outerloss(\hatatnext(\wtnext))^\top (\hatatnextnext(\wtnext) - \hatatnext(\wtnext)) + \frac{L}{2} \left\| \hatatnextnext(\wtnext) - \hatatnext(\wtnext) \right\|^2 \\
        \le & \left\| \gradhata \outerloss(\hatatnext(\wtnext)) \right\| \cdot \left\| \hatatnextnext(\wtnext) - \hatatnext(\wtnext) \right\| + \frac{L}{2} \left\| \hatatnextnext(\wtnext) - \hatatnext(\wtnext) \right\|^2 \\
        \le & \rho \cdot \left\| -\innerlrtnext\gradhata\innerloss(\hatatnext) \right\| + \frac{L}{2} \left\| -\innerlrtnext\gradhata\innerloss(\hatatnext) \right\|^2 \\
        \le & \innerlrtnext \rho^2 + \frac{L}{2} \innerlrtnext^2 \rho^2,
    \end{aligned}
    \label{eq:th1-2}
    \end{equation}
    where $\hatatnextnext(\wtnext) - \hatatnext(\wtnext) = -\innerlrtnext\gradhata\innerloss(\hatatnext)$, the first inequality holds for Lemma~\ref{lem:lipschitz_upper_bound}, the second inequality holds for Cauchy-Schwarz inequality, the third inequality holds for $\left\| \gradhata\outerloss(\hatatnext(\wtnext)) \right\| \le \rho$, and the last inequality holds for $\left\| \gradhata\innerloss(\hatatnext) \right\| \le \rho$.
    It can be proved that the gradient of $\omega$ with respect to $\outerloss$ is Lipschitz continuous and we assume the Lipschitz constant is $L$.
    Therefore, for the second term,
    \begin{equation}
    \begin{aligned}
            & \outerloss(\hatatnext(\wtnext)) - \outerloss(\hatatnext(\wt)) \\
        \le & \gradw\outerloss(\hatatnext(\wt))^\top (\wtnext - \wt) + \frac{L}{2} \left\| \wtnext - \wt \right\|^2 \\
        =   & -\outerlrt \gradw \outerloss(\hatatnext(\wt))^\top \gradw \outerloss(\hatatnext(\wt)) + \frac{L\outerlrt^2}{2} \left\| \gradw \outerloss(\hatatnext(\wt)) \right\|^2 \\
        =   & -(\outerlrt - \frac{L\outerlrt^2}{2}) \left\| \gradw \outerloss(\hatatnext(\wt)) \right\|^2,
    \end{aligned}
    \label{eq:th1-3}
    \end{equation}
    where $\wtnext - \wt = -\outerlrt \gradw \outerloss(\hatatnext(\wt))$, and the first inequality holds for Lemma~\ref{lem:lipschitz_upper_bound}. Therefore,~(\ref{eq:th1-1}) becomes
    \begin{equation}
    \begin{aligned}
        \outerloss(\hatatnextnext(\wtnext)) - \outerloss(\hatatnext(\wt)) \le \innerlrtnext \rho^2 + \frac{L}{2} \innerlrtnext^2 \rho^2 - (\outerlrt - \frac{L\outerlrt^2}{2}) \left\| \gradw \outerloss(\hatatnext(\wt)) \right\|^2.
    \end{aligned}
    \label{eq:th1-4}
    \end{equation}
    Rearranging the terms of~(\ref{eq:th1-4}), we obtain
    \begin{equation}
    \begin{aligned}
        (\outerlrt - \frac{L\outerlrt^2}{2}) \left\| \gradw \outerloss(\hatatnext(\wt)) \right\|^2 \le \outerloss(\hatatnext(\wt)) - \outerloss(\hatatnextnext(\wtnext)) + \innerlrtnext \rho^2 + \frac{L}{2} \innerlrtnext^2 \rho^2.
    \end{aligned}
    \label{eq:th1-5}
    \end{equation}
    Then, we sum up both sides of~(\ref{eq:th1-5}),
    \begin{equation}
    \begin{aligned}
            & \sum_{t=1}^T (\outerlrt - \frac{L\outerlrt^2}{2}) \left\| \gradw \outerloss(\hatatnext(\wt)) \right\|^2 \\
        \le & \outerloss(\hat{\alpha}_2(\omega_1)) - \outerloss(\hat{\alpha}_{T+2}(\omega_{T+1})) + \sum_{t=1}^T (\innerlrtnext \rho^2 + \frac{L}{2} \innerlrtnext^2 \rho^2) \\
        \le & \outerloss(\hat{\alpha}_2(\omega_1)) + \sum_{t=1}^T (\innerlrtnext \rho^2 + \frac{L}{2} \innerlrtnext^2 \rho^2).
    \end{aligned}
    \label{eq:th1-6}
    \end{equation}
    Therefore,
    \begin{equation}
    \begin{aligned}
            & \min_{1 \le t \le T} \mathbb{E}\left[ \left\| \gradw \outerloss(\hatatnext(\wt)) \right\|^2 \right] \\
        \le & \frac{\sum_{t=1}^T (\outerlrt - \frac{L\outerlrt^2}{2}) \left\| \gradw \outerloss(\hatatnext(\wt)) \right\|^2}{\sum_{t=1}^T (\outerlrt - \frac{L\outerlrt^2}{2})} \\
        \le & \frac{1}{\sum_{t=1}^T (2\outerlrt - L\outerlrt^2)} \left[ 2\outerloss(\hat{\alpha}_2(\omega_1)) + \sum_{t=1}^T (2\innerlrtnext \rho^2 + L\innerlrtnext^2 \rho^2) \right] \\
        \le & \frac{1}{\sum_{t=1}^T \outerlrt} \left[ 2\outerloss(\hat{\alpha}_2(\omega_1)) + \sum_{t=1}^T \innerlrtnext \rho^2(2 + L\innerlrtnext) \right] \\
        \le & \frac{1}{T\outerlrt} \left[ 2\outerloss(\hat{\alpha}_2(\omega_1)) + T \innerlrtnext \rho^2(2 + L) \right] \\
        =   & \frac{2\outerloss(\hat{\alpha}_2(\omega_1))}{T\outerlrt} + \frac{\innerlrtnext \rho^2 (2 + L)}{\outerlrt} \\
        =   & \frac{2\outerloss(\hat{\alpha}_2(\omega_1))}{T}\max\{L, \frac{\sqrt{T}}{c_2}\} + \min\{1, \frac{c_1}{T}\} \max\{L, \frac{\sqrt{T}}{c_2}\} \rho^2 (2 + L) \\
        \le & \frac{2\outerloss(\hat{\alpha}_2(\omega_1))}{c_2\sqrt{T}} + \frac{c_1 \rho^2 (2 + L)}{c_2\sqrt{T}} \\
        =   &\mathcal{O}\left(\frac{1}{\sqrt{T}}\right),
    \end{aligned}
    \label{eq:th1-7}
    \end{equation}
    where the second inequality holds according to~(\ref{eq:th1-6}), the third inequality holds for $\sum_{t=1}^T \left( 2\outerlrt - L\outerlrt^2 \right) \ge \sum_{t=1}^T \outerlrt$.
\end{proof}

\subsection{Theorem 2: Convergence of the Inner Loss}
\begin{lemma} \label{lem:sequence_convergence} (Lemma A.5 in~\citet{mairal2013stochastic})
    Let $(a_n)_{n \ge 1}, (b_n)_{n \ge 1}$ be two non-negative real sequences such that the series $\sum_{n=1}^{\infty} a_n$ diverges, the series $\sum_{n=1}^{\infty} a_n b_n$ converges, and there exists $C>0$ such that $\left|b_{n+1} - b_n\right| \le C a_n$. Then, the sequence $(b_n)_{n \ge 1}$ converges to 0.
\end{lemma}

\begin{theorem}
    Suppose $\outerloss$ is Lipschitz-smooth with constant L, the gradient of $\outerloss$ and $\innerloss$ is bounded by $\rho$. Let the training iterations be $T$, the inner-level optimization learning rate $\innerlrt=\min\{1, \frac{c_1}{T}\}$ for some constant $c_1 > 0$ where $\frac{c_1}{T} < 1$. Let the outer-level optimization learning rate $\outerlrt=\min\{\frac{1}{L}, \frac{c_2}{\sqrt{T}}\}$ for some constant $c_2 > 0$ where $c_2 \le \frac{\sqrt{T}}{L}$, and $\sum_{t=1}^\infty\outerlrt \le \infty, \sum_{t=1}^\infty\outerlrt^2 \le \infty$. $\innerloss$ achieves
    \begin{equation}
        \lim_{t \to \infty} \E \left[ \left\| \grada \innerloss(\at; \wt) \right\|^2 \right] = 0.
    \end{equation}
\end{theorem}

\begin{proof}
    First,
    \begin{equation}
    \begin{aligned}
          & \innerloss(\atnext; \wtnext) - \innerloss(\at; \wt) \\
        = & \left\{ \innerloss(\atnext; \wtnext) - \innerloss(\atnext; \wt) \right\} + \left\{ \innerloss(\atnext; \wt) - \innerloss(\at; \wt) \right\}.
    \end{aligned}
    \label{eq:th2-1}
    \end{equation}
    For the first term in~(\ref{eq:th2-1}),
    \begin{equation}
    \begin{aligned}
            & \innerloss(\atnext; \wtnext) - \innerloss(\atnext; \wt) \\
        \le & \gradw\innerloss(\atnext; \wt)^\top (\wtnext - \wt) + \frac{L}{2} \left\| \wtnext - \wt \right\|^2 \\
        =   & - \outerlrt \gradw\innerloss(\atnext; \wt)^\top \gradw\outerloss(\atnext(\wt)) + \frac{L\outerlrt^2}{2} \left\| \gradw\outerloss(\atnext(\wt)) \right\|^2.
    \end{aligned}
    \label{eq:th2-2}
    \end{equation}
    where $\wtnext - \wt = -\outerlrt \gradw\outerloss(\atnext(\wt))$, and the first inequality holds according to Lemma~\ref{lem:lipschitz_upper_bound}.
    For the second term in~(\ref{eq:th2-1}),
    \begin{equation}
    \begin{aligned}
            & \innerloss(\atnext; \wt) - \innerloss(\at; \wt) \\
        \le & \grada\innerloss(\at; \wt)^\top (\atnext - \at) + \frac{L}{2} \left\| \atnext - \at \right\|^2 \\
        =   & - \innerlrt \grada\innerloss(\at; \wt)^\top \grada\innerloss(\at; \wt) + \frac{L\innerlrt^2}{2} \left\| \grada\innerloss(\at; \wt) \right\|^2 \\
        =   & - (\innerlrt - \frac{L\innerlrt^2}{2}) \left\| \grada\innerloss(\at; \wt) \right\|^2.
    \end{aligned}
    \label{eq:th2-3}
    \end{equation}
    where $\atnext - \at = -\innerlrt \grada\innerloss(\at; \wt)$, and the first inequality holds according to Lemma~(\ref{lem:lipschitz_upper_bound}).
    Therefore,~(\ref{eq:th2-1}) becomes
    \begin{equation}
    \begin{aligned}
            & \innerloss(\atnext; \wtnext) - \innerloss(\at; \wt) \\
        \le & - \outerlrt \gradw\innerloss(\atnext; \wt)^\top \gradw\outerloss(\atnext(\wt)) + \frac{L\outerlrt^2}{2} \left\| \gradw\outerloss(\atnext(\wt)) \right\|^2 \\
            & - (\innerlrt - \frac{L\innerlrt^2}{2}) \left\| \grada\innerloss(\at; \wt) \right\|^2.
    \end{aligned}
    \label{eq:th2-4}
    \end{equation}
    Taking expectation of both sides of~(\ref{eq:th2-4}) and rearranging the terms, we obtain
    \begin{equation}
    \begin{aligned}
            & \innerlrt \expectation{\squaredltwonorm{\grada\innerloss(\at; \wt)}} + \outerlrt \expectation{\ltwonorm{\gradw\innerloss(\atnext; \wt)} \cdot \ltwonorm{\gradw\outerloss(\atnext(\wt))}} \\
        \le & \expectation{\innerloss(\at; \wt)} - \expectation{\innerloss(\atnext; \wtnext)} + \frac{L\outerlrt^2}{2} \expectation{\squaredltwonorm{\gradw\outerloss(\atnext(\wt))}} \\
            & + \frac{L\innerlrt^2}{2} \expectation{\squaredltwonorm{\grada\innerloss(\at; \wt)}}.
    \end{aligned}
    \label{eq:th2-5}
    \end{equation}
    Summing up both sides of~(\ref{eq:th2-5}) from $t=1$ to $\infty$,
    \begin{equation}
    \begin{aligned}
            & \sum_{t=1}^\infty \innerlrt \expectation{\squaredltwonorm{\grada\innerloss(\at; \wt)}} + \sum_{t=1}^\infty \outerlrt \expectation{\ltwonorm{\gradw\innerloss(\atnext; \wt)} \cdot \ltwonorm{\gradw\outerloss(\atnext(\wt))}} \\
        \le & \expectation{\innerloss(\alpha_1; \omega_1)} - \lim_{t\to\infty}\expectation{\innerloss(\atnext; \wtnext)} + \sum_{t=1}^\infty \frac{L\outerlrt^2}{2} \expectation{\squaredltwonorm{\gradw\outerloss(\atnext(\wt))}} \\
            & + \sum_{t=1}^\infty \frac{L\innerlrt^2}{2} \expectation{\squaredltwonorm{\grada\innerloss(\at; \wt)}} \\
        \le & \sum_{t=1}^\infty \frac{L(\innerlrt^2 + \outerlrt^2)\rho^2}{2} + \expectation{\innerloss(\alpha_1; \omega_1)} \le \infty,
    \end{aligned}
    \label{eq:th2-6}
    \end{equation}
    where the second inequality holds for $\sum_{t=1}^\infty \innerlrt^2 \le \infty$, $\sum_{t=1}^\infty \outerlrt^2 \le \infty$, $\ltwonorm{\grada\innerloss(\at; \wt)} \le \rho$, $\ltwonorm{\gradw\outerloss(\atnext(\wt))} \le \rho$.
    Since
    \begin{equation}
        \sum_{t=1}^\infty \outerlrt \expectation{\ltwonorm{\gradw\innerloss(\atnext; \wt)} \cdot \ltwonorm{\gradw\outerloss(\atnext(\wt))}} \le L\rho \sum_{t=1}^\infty \outerlrt \le \infty.
    \label{eq:th2-7}
    \end{equation}
    Therefore, we have
    \begin{equation}
        \sum_{t=1}^\infty \innerlrt \expectation{\squaredltwonorm{\grada\innerloss(\at; \wt)}} < \infty.
    \label{eq:th2-8}
    \end{equation}
    Since $\left|(\|a\|+\|b\|)(\|a\|-\|b\|)\right| \le \|a+b\|\|a-b\|$, we can derive that
    \begin{equation}
    \begin{aligned}
            & \left| \expectation{\squaredltwonorm{\grada\innerloss(\atnext; \wtnext)}} - \expectation{\squaredltwonorm{\grada\innerloss(\at; \wt)}} \right| \\
        =   & \Big| \Bigexpectation{\big( \ltwonorm{\grada\innerloss(\atnext; \wtnext)} + \ltwonorm{\grada\innerloss(\at; \wt)} \big) + \big( \ltwonorm{\grada\innerloss(\atnext; \wtnext)} - \ltwonorm{\grada\innerloss(\at; \wt)} \big)} \Big| \\
        \le & \Bigexpectation{\Big| \ltwonorm{\grada\innerloss(\atnext; \wtnext)} + \ltwonorm{\grada\innerloss(\at; \wt)} \Big| \Big| \ltwonorm{\grada\innerloss(\atnext; \wtnext)} - \ltwonorm{\grada\innerloss(\at; \wt)} \Big|} \\
        \le & \Bigexpectation{\ltwonorm{\grada\innerloss(\atnext; \wtnext) + \grada\innerloss(\at; \wt)} \cdot \ltwonorm{\grada\innerloss(\atnext; \wtnext) - \grada\innerloss(\at; \wt)}} \\
        \le & \Bigexpectation{\big( \ltwonorm{\grada\innerloss(\atnext; \wtnext)} + \ltwonorm{\grada\innerloss(\at; \wt)} \big) \ltwonorm{\grada\innerloss(\atnext; \wtnext) - \grada\innerloss(\at; \wt)}} \\
        \le & 2L\rho \Bigexpectation{\ltwonorm{\left(\atnext, \wtnext\right) - \left(\at, \wt\right)}} \\
        \le & 2L\rho\innerlrt\outerlrt \Bigexpectation{\ltwonorm{\left( \grada\innerloss(\at; \wt), \gradw\outerloss(\atnext(\wt)) \right)}} \\
        \le & 2L\rho\innerlrt\outerlrt \sqrt{\expectation{\squaredltwonorm{\grada\innerloss(\at; \wt)}} + \expectation{\squaredltwonorm{\gradw\outerloss(\atnext(\wt))}}} \\
        \le & 2L\rho\innerlrt\outerlrt \sqrt{2\rho^2} \\
        \le & 2\sqrt{2}L\rho^2\innerlrt\outerlrt.
    \end{aligned}
    \label{eq:th2-9}
    \end{equation}
    Since $\sum_{t=1}^\infty \innerlrt = \infty$, according to Lemma~\ref{lem:sequence_convergence}, we have
    \begin{equation}
        \lim_{t \to \infty} \E \left[ \left\| \grada \innerloss(\at; \wt) \right\|^2 \right] = 0.
    \label{eq:th2-10}
    \end{equation}
\end{proof}

\section{Experimental Details}\label{appendix:exp_detail}
In this section, we provide a concrete description of our experiments and detailed hyper-parameters of RAT. For each run of experiments, we run on a single Nvidia Tesla V100 GPUs and 16 CPU cores (Intel Xeon Gold 6230 CPU @ 2.10GHz) for training.

\subsection{Tasks}
In phase one of our experiments, we evaluate our method on eight robotic manipulation tasks obtained from Meta-world~\citep{yu2020meta}. These tasks serve as a representative set for testing the effectiveness of our approach. In phase two, we further assess our method on two locomotion tasks sourced from Mujoco~\citep{6386109}. By including tasks from both domains, we aim to demonstrate the versatility and generalizability of our approach across different task types. The specific tasks we utilize in our experiments are as follows:

\noindent \textbf{Meta-world}
\begin{itemize} [leftmargin=10pt]
    \item Door Lock: An agent controls a simulated Sawyer arm to lock the door.
    \item Door Unlock: An agent controls a simulated Sawyer arm to unlock the door.
    \item Drawer Open: An agent controls a simulated Sawyer arm to open the drawer to a target position.
    \item Drawer Close: An agent controls a simulated Sawyer arm to close the drawer to a target position.
    \item Faucet Open: An agent controls a simulated Sawyer arm to open the faucet to a target position.
    \item Faucet Close: An agent controls a simulated Sawyer arm to close the faucet to a target position.
    \item Window Open: An agent controls a simulated Sawyer arm to open the window to a target position.
    \item Window Close: An agent controls a simulated Sawyer arm to close the window to a target position.
\end{itemize}

\noindent \textbf{Mujoco}
\begin{itemize} [leftmargin=10pt]
    \item Half Cheetah: A 2d robot with nine links and eight joints aims to learn to run forward (right) as fast as possible.
    \item Walker: A 2d two-legged robot aims to move in the forward (right).
\end{itemize}

\subsection{Hyper-parameters Setting}
In our experiments, we adopt the PEBBLE~\cite{lee2021pebble} as our baseline approach for reward learning from human feedback. It is worth to emphasizes that the PA-AD (oracle)~\citep{zhang2021robust} and SA-RL (oracle)~\citep{sun2022who} use the truth victim reward function. To ensure a fair comparison, All methods employ the same neural network structure and keep the same parameter settings as described in their work. The specific hyper-parameters for SA-RL are provided in Table~\ref{tab:hyper_sarl}.

\begin{table}[!htbp]
\begin{center}
{
\caption{Hyper-parameters of RAT for adversary training.}
\label{tab:hyper_BATTLE}
\begin{tabular}{ll|ll}
\toprule
\textbf{Hyper-parameter} & \textbf{Value} & \textbf{Hyper-parameter} & \textbf{Value} \\
\midrule
Number of layers           & $3$        & Hidden units of each layer & $256$ \\
Learning rate              & $0.0003$   & Batch size                 & $1024$  \\
Length of segment          & $50$       & Number of reward functions  & $3$ \\
Frequency of feedback      & $5000$     & Feedback batch size & $128$ \\
Adversarial budget         & $0.1$      & $(\beta_1, \beta_2)$ & $(0.9, 0.999)$ \\
\bottomrule
\end{tabular}
}
\end{center}
\end{table}

\begin{table}[!htbp]
\begin{center}
{
\caption{Hyper-parameters of SA-RL for adversary training.}
\label{tab:hyper_sarl}
\begin{tabular}{ll|ll}
\toprule
\textbf{Hyper-parameter} & \textbf{Value} & \textbf{Hyper-parameter} & \textbf{Value} \\
\midrule
Number of layers           & $3$        & Hidden units of each layer  & $256$   \\
Learning rate              & $0.00005$  & Mini-Batch size             & $32$    \\
Length of segment          & $50$       & Number of reward functions  & $3$     \\
Frequency of feedback      & $5000$     & Feedback batch size         & $128$   \\
Adversarial budget         & $0.1$      & Entropy coefficient         & $0.0$   \\
Clipping parameter         & $0.2$      & Discount $\gamma$           & $0.99$  \\ 
GAE lambda                 & $0.95$     & KL divergence target       & $0.01$   \\
\bottomrule
\end{tabular}
}
\end{center}
\end{table}

\subsection{Victim Agents Settings}\label{appendix:victim_agents_settings}
Our experiment is divided into two phases. In the first phase, we conduct experiments using a variety of simulated robotic manipulation tasks from the Meta-world environment. In the second phase, we shift our focus to two continuous control environments from the OpenAI Gym MuJoCo suite.

\begin{figure}[!ht]
    \centering
    \includegraphics[width=0.7\linewidth]{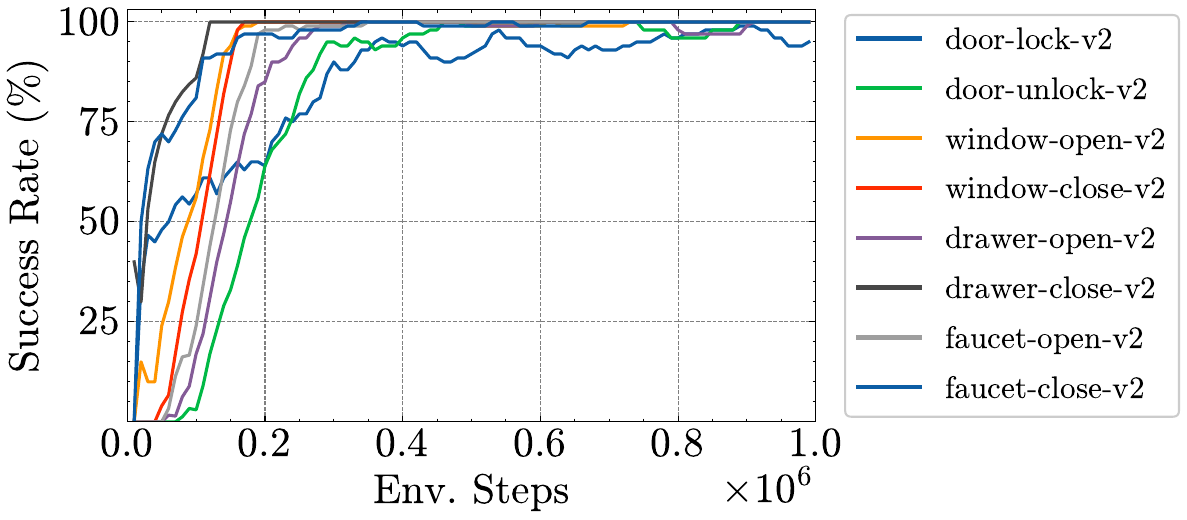}
    \caption{The evaluation curves for the training of victim agents are measured based on their success rate well-designed in Meta-world~\cite{yu2020meta}.}
    \label{fig:victim_success_rate}
\end{figure}

\noindent \textbf{Meta-world.} 
The victim models for Meta-world tasks are trained using the Soft Actor-Critic (SAC) algorithm, as introduced by~\citet{haarnoja2018soft}. Implementation is based on the open-source repository available at~\footnote{\url{https://github.com/denisyarats/pytorch_sac}}. 
In each agent's training, fully connected neural networks are utilized both as the policy network and for the double Q networks. The specific hyperparameters employed in our experiments are detailed in Table~\ref{tab:hyper_sac}. As depicted in Figure~\ref{fig:victim_success_rate}, each victim agent has been thoroughly trained to master a specific manipulation skill.

\begin{table}[!htbp]
\begin{center}
{
\caption{Hyper-parameters of SAC for victim training.}
\label{tab:hyper_sac}
\begin{tabular}{ll|ll}
\toprule
\textbf{Hyper-parameter} & \textbf{Value} & \textbf{Hyper-parameter} & \textbf{Value} \\
\midrule
Total training steps       & ${10}^6$   & Replay buffer capacity & ${10}^6$ \\ 
Number of layers           & $3$        & Initial temperature & $0.1$ \\ 
Hidden units of each layer & $256$      & Optimizer & Adam \\
Learning rate              & $0.0001$   & Critic target update freq & $2$ \\
Discount $\gamma$          & $0.99$     & Critic EMA $\tau$ & $0.005$ \\
Batch size                 & $1024$     & $(\beta_1, \beta_2)$ & $(0.9, 0.999)$ \\
Random steps               & $5000$     & Agent update frequency & $1$ \\ 
\bottomrule
\end{tabular}
}
\end{center}
\end{table}

\noindent \textbf{Mujoco.}
To demonstrate the vulnerability of the Decision Transformer, we employ well-trained models of expert-level proficiency. Specifically, we utilize the Cheetah agent\footnote{\url{https://huggingface.co/edbeeching/decision-transformer-gym-halfcheetah-expert}} and the Walker agent\footnote{\url{https://huggingface.co/edbeeching/decision-transformer-gym-walker2d-expert}}, both of which are based on Decision Transformer~\cite{chen2021decision} models. These models have been trained on expert trajectories sampled from the Gym environment.

\subsection{Scenario Design}\label{appendix:scenario_design}
To assess the efficacy of our method, we meticulously crafted two experimental setups: the Manipulation Scenario and the Opposite Behavior Scenario.

\noindent \textbf{Scenario Description.} In both scenarios, the victim agent is a proficiently trained policy in robotic tasks, as detailed in ~\ref{appendix:victim_agents_settings}. In the Manipulation Scenario, the adversary's aim is to alter the agent's behavior via targeted adversarial attacks, compelling the agent to grasp objects distant from the initially intended target location. The successful completion of these grasping actions signifies the effectiveness of the adversarial attack. Conversely, in the Opposite Behavior Scenario, the victim policy is a well-established policy in simulated robotic manipulation tasks. Here, the adversary's objective is to manipulate the agent's behavior to perform actions contrary to its original purpose. For example, if the policy is originally designed to open windows, the attacker endeavors to deceive the agent into closing them instead.

\begin{table}[!h]
\centering
\caption{Success metrics for the Meta-world tasks are quantified in meters. The metrics for the first four rows are sourced from~\cite{yu2020meta}, and we utilize the built-in functions provided therein without any alterations. The Manipulation Scenario metric, devised by us, is applied across all tasks within the Manipulation Scenario.}
\label{tab:task_metrics}
\begin{tabular}{ll|ll}
\toprule
\footnotesize \textbf{Task} & \textbf{Success Metric} & \textbf{Task} & \textbf{Success Metric} \\
\midrule
door-lock   & $\I_{\|o - t\|_2 < 0.02}$  & door-unlock & $\I_{\|o - t\|_2 < 0.02}$   \\
drawer-open & $\I_{\|o - t\|_2 < 0.03}$  & drawer-close & $\I_{\|o - t\|_2 < 0.055}$ \\
faucet-open &  $\I_{\|o - t\|_2 < 0.07}$ & faucet-close & $\I_{\|o - t\|_2 < 0.07}$  \\
window-open & $\I_{\|o - t\|_2 < 0.05}$  & window-close & $\I_{\|o - t\|_2 < 0.05}$  \\
\midrule
Manipulation Scenario & $\I_{\|o - t\|_2 < 0.05}$  &    &                            \\
\bottomrule
\end{tabular}
\end{table}

\noindent \textbf{Evaluation Metric.} The success metric for all our tasks revolves around the proximity between the task-relevant object and its final goal position, denoted as $\I_{\|o - t\|_2 < \epsilon}$, where $\epsilon$ is a minimal distance threshold, such as 5 cm. In the Manipulation Scenario, we set $\epsilon=0.05$ (5cm). For the Opposite Behavior Scenario, we apply the success metrics and thresholds specified for each task by Meta-world~\citep{yu2020meta}. We summarize the all success metrics in our experiments in the Table~\ref{tab:task_metrics}.

\newpage
\section{Full Experiments}\label{appendix:extensive_exp}
\subsection{Full Experiment Results}\label{appendix:full_experiment_results}

\begin{figure*}[!ht]
\vspace{-0.5em}
\centering
\begin{center}
\includegraphics[width=0.9\linewidth]{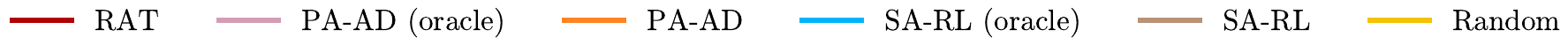}
\vspace{-1em}
\end{center}
\begin{tabular}{cccc}
\hspace*{-0.5em} \subfloat[Door Lock]{\includegraphics[width=0.25\linewidth]{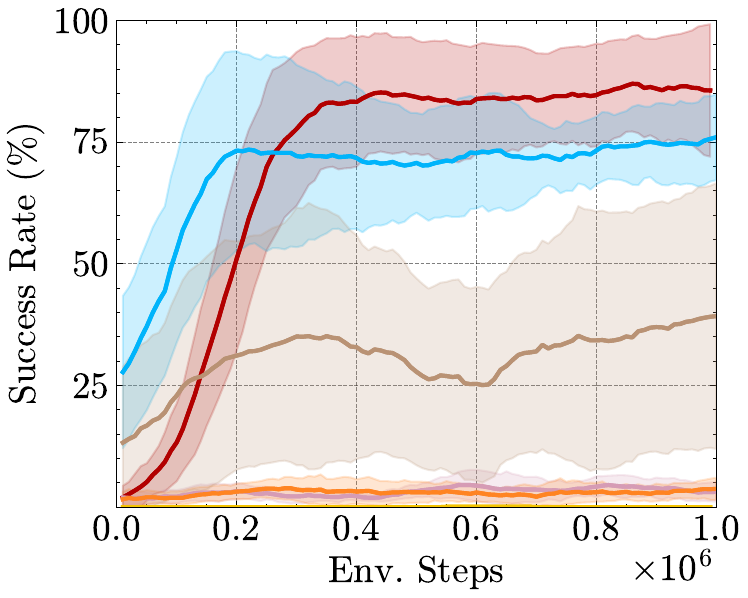}}
& \hspace*{-1.6em} \subfloat[Drawer Open]{\includegraphics[width=0.25\linewidth]{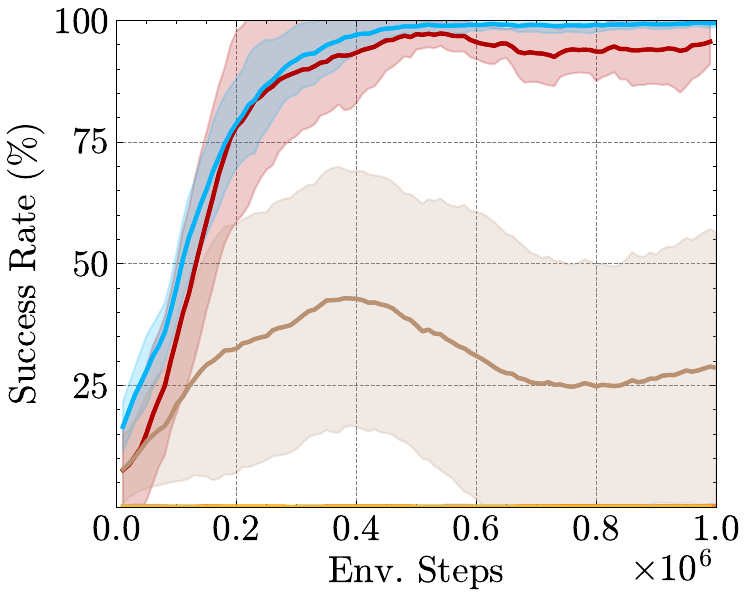}}
& \hspace*{-1.6em} \subfloat[Faucet Open]{\includegraphics[width=0.25\linewidth]{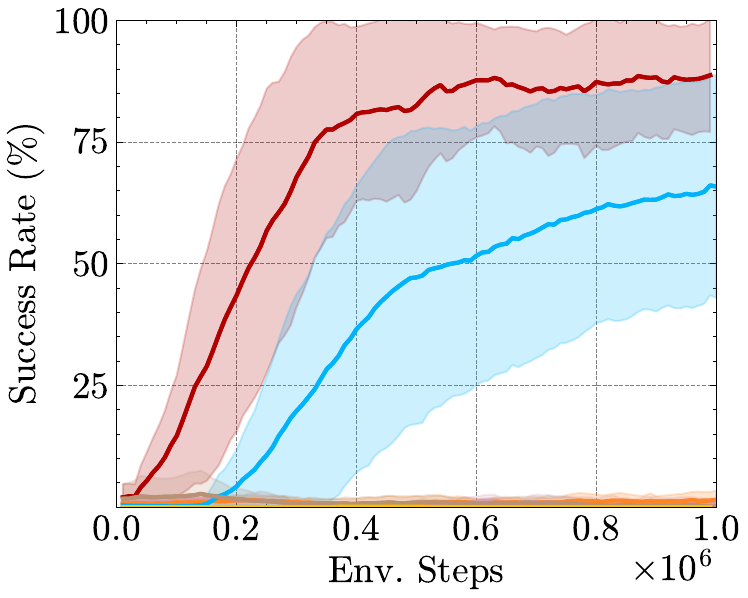}}
& \hspace*{-1.6em} \subfloat[Window Open]{\includegraphics[width=0.25\linewidth]{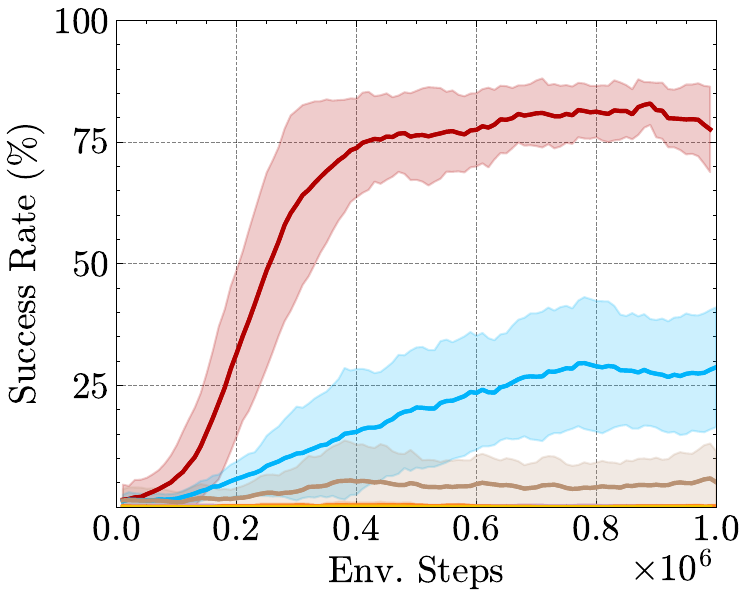}}
\vspace*{-0.7em}
\\
\hspace*{-0.5em} \subfloat[Door Unlock]{\includegraphics[width=0.25\linewidth]{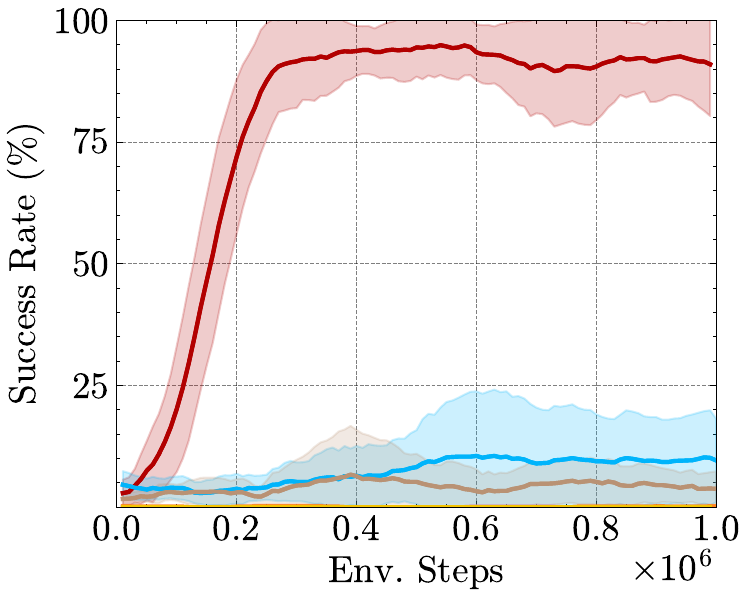}}
& \hspace*{-1.6em} \subfloat[Drawer Close]{\includegraphics[width=0.25\linewidth]{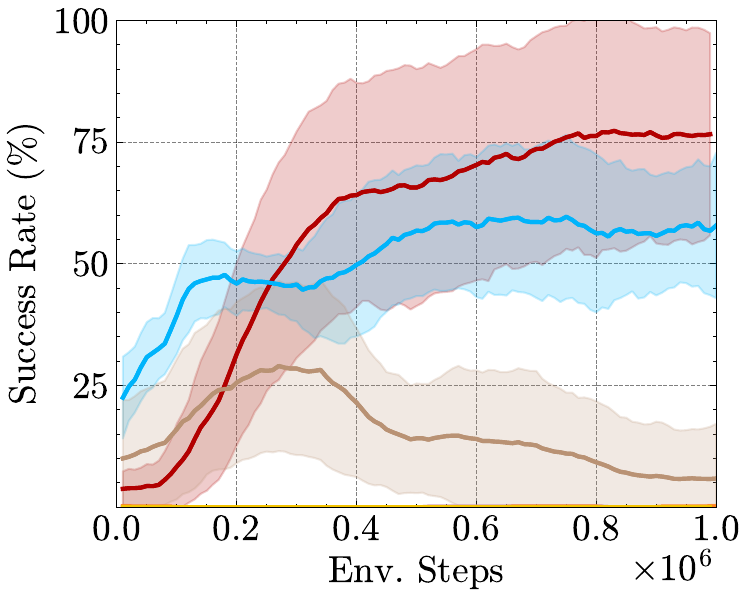}}
& \hspace*{-1.6em} \subfloat[Faucet Close]{\includegraphics[width=0.25\linewidth]{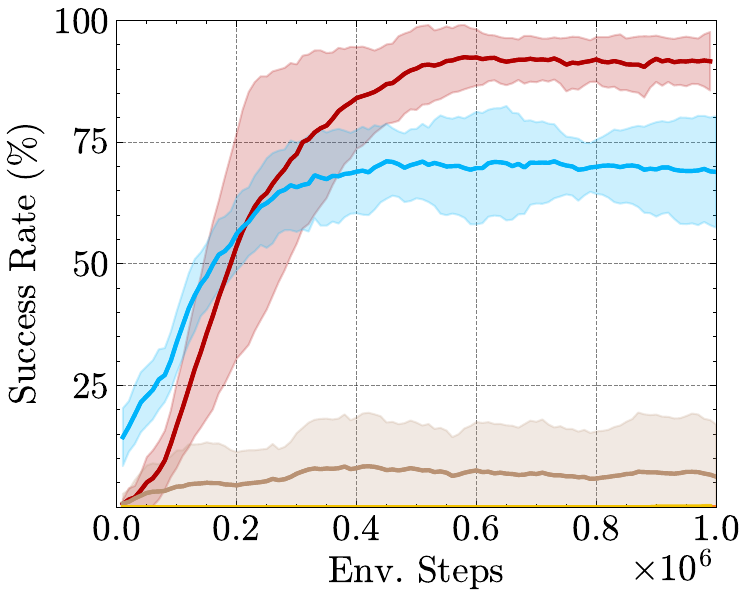}}
& \hspace*{-1.6em} \subfloat[Window Close]{\includegraphics[width=0.25\linewidth]{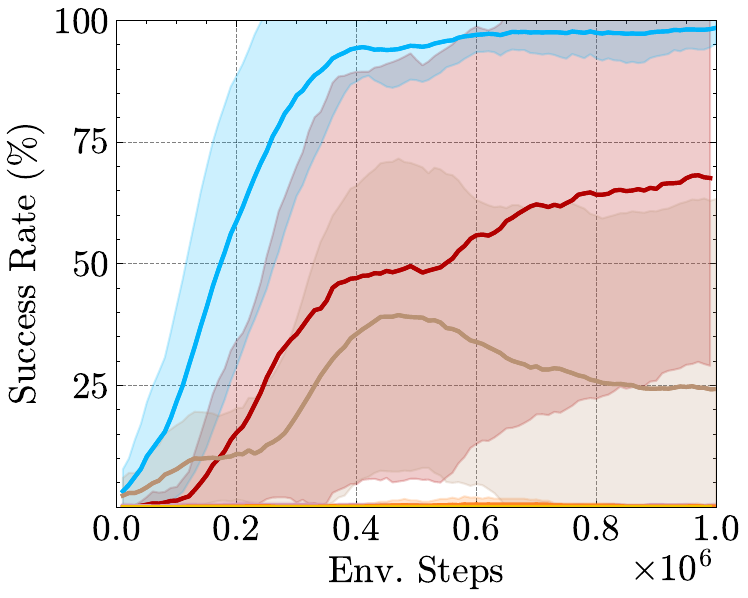}}
\end{tabular}
\caption{Training curves of different methods on various tasks in the manipulation scenario. The solid line and shaded area denote the mean and the standard deviation of success rate, respectively, over ten runs. The red line (our method) outperforms all the baselines in PbRL setting and even exceeds most baselines in oracle setting.}
\label{fig:main_adv}
\end{figure*}

\begin{figure*}[!ht]
\centering
\begin{center}
\includegraphics[width=0.9\linewidth]{main_header.pdf}
\vspace{-1em}
\end{center}
\begin{tabular}{cccc}
\hspace*{-0.5em} \subfloat[Door Lock]{\includegraphics[width=0.25\linewidth]{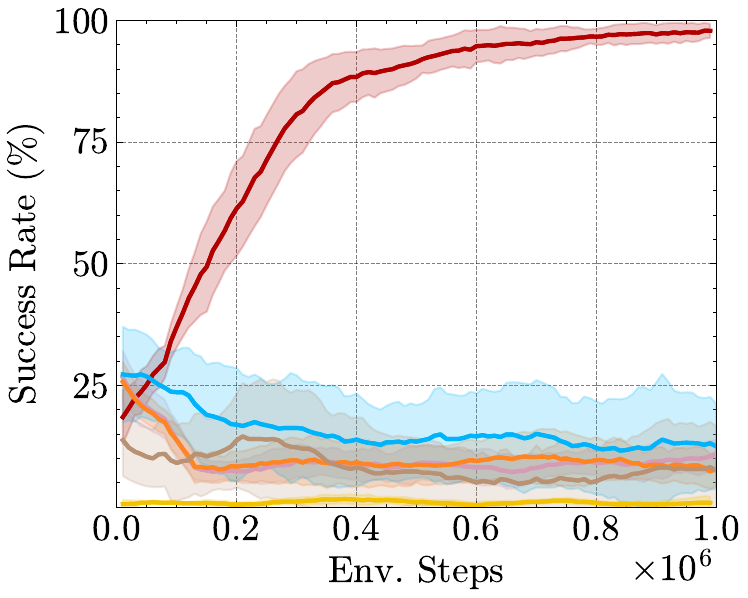}}
& \hspace*{-1.6em} \subfloat[Drawer Open]{\includegraphics[width=0.25\linewidth]{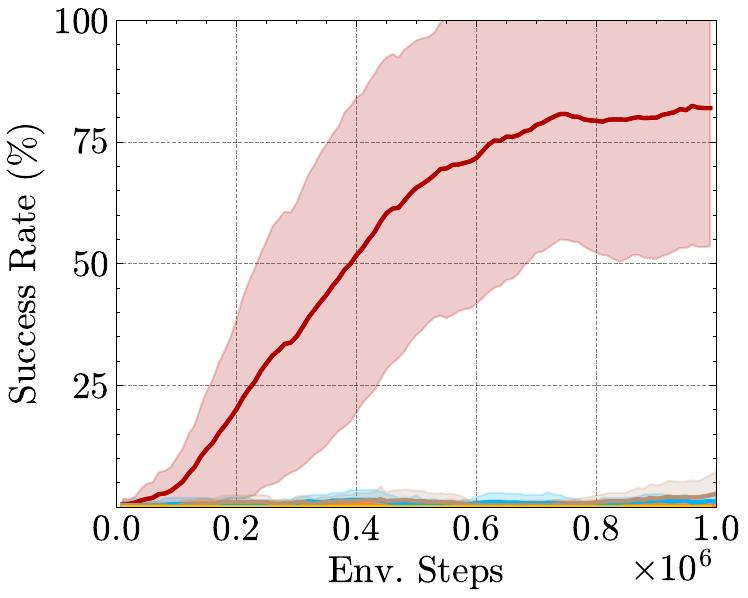}}
& \hspace*{-1.6em} \subfloat[Faucet Open]{\includegraphics[width=0.25\linewidth]{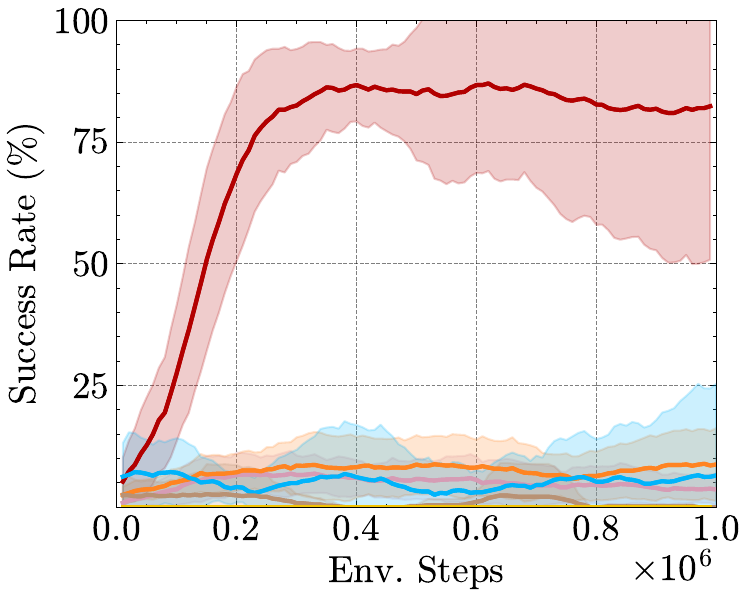}}
& \hspace*{-1.6em} \subfloat[Window Open]{\includegraphics[width=0.25\linewidth]{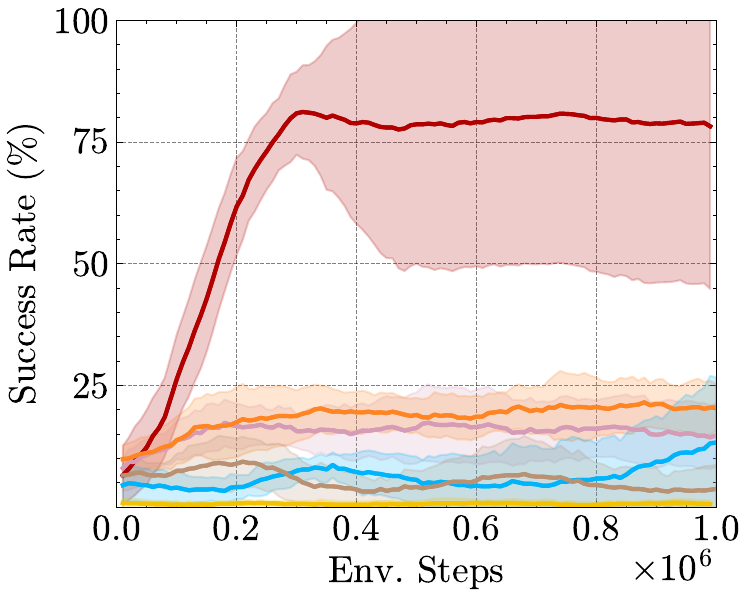}} 
\vspace*{-0.7em} 
\\
\hspace*{-0.5em} \subfloat[Door Unlock]{\includegraphics[width=0.25\linewidth]{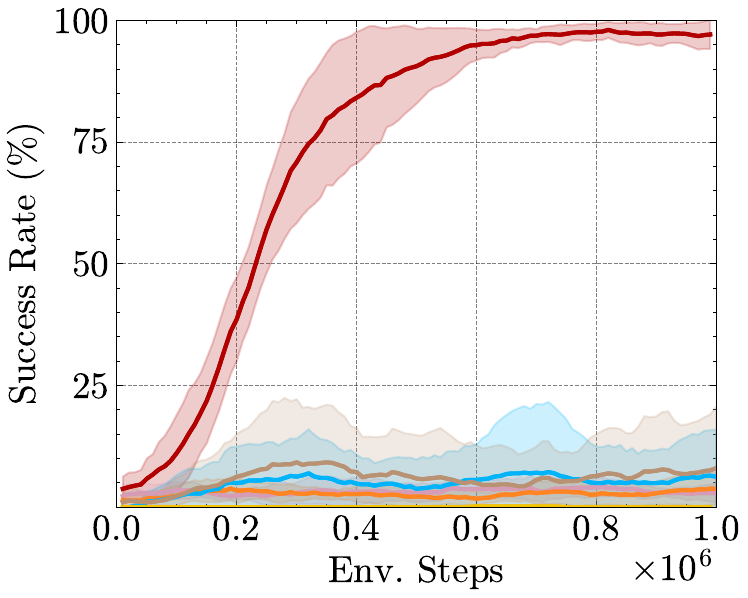}}
& \hspace*{-1.6em} \subfloat[Drawer Close]{\includegraphics[width=0.25\linewidth]{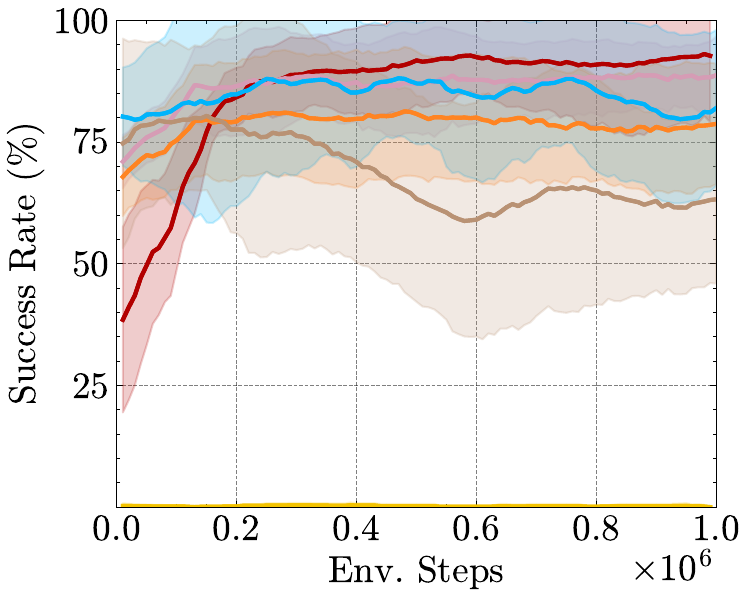}}
& \hspace*{-1.6em} \subfloat[Faucet Close]{\includegraphics[width=0.25\linewidth]{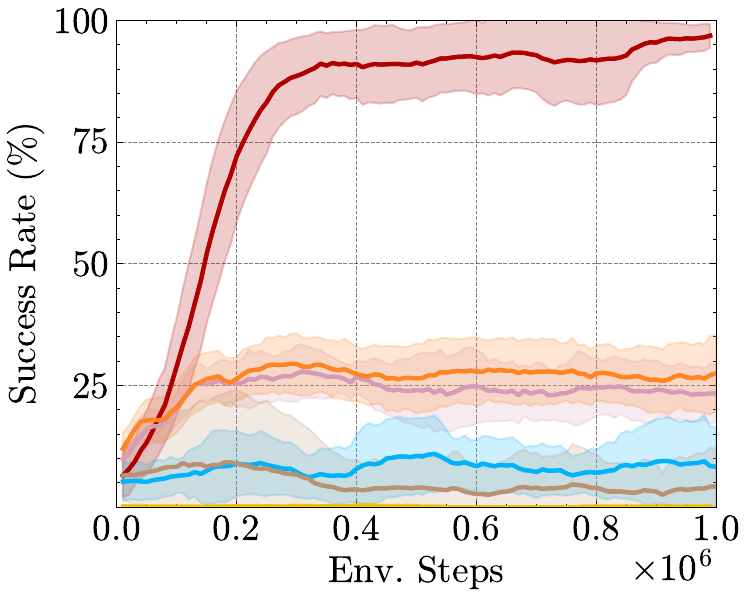}}
& \hspace*{-1.6em} \subfloat[Window Close]{\includegraphics[width=0.25\linewidth]{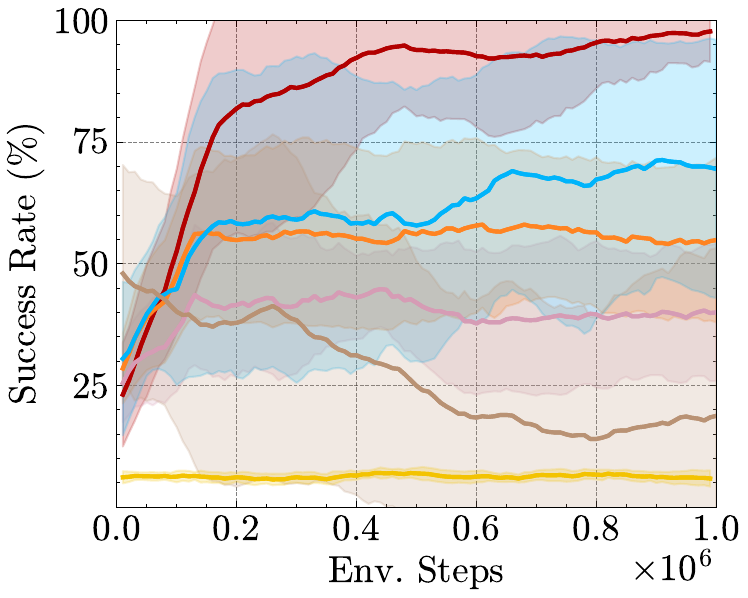}}
\end{tabular}
\caption{Training curves of all methods on various tasks in the opposite behaviors scenario. The solid line and shaded area denote the mean and the standard deviation of success rate over ten runs. In this scenario, the red line (our method) outperforms all the baselines in both PbRL setting and oracle setting, which demonstrates the effectiveness of \ourmethod.}
\label{fig:main_oppo}
\vspace{-1.0em}
\end{figure*}

\subsection{Ablation studies}
\noindent \textbf{Impact of Feedback Amount.}
We evaluate the performance of \ourmethod using different numbers of preference labels. Table~\ref{tab:abla_feedback} presents the results across varying numbers of labels: ${3000, 5000, 7000, 9000}$ for the Drawer Open task in the manipulation scenario and ${1000, 3000, 5000, 7000}$ for the Faucet Close task in the opposite behavior scenario. The experimental results demonstrate that increasing the number of human feedback labels significantly improves the performance of \ourmethod, leading to a stronger adversary and a more stable attack success rate. For instance, in the Drawer Open task, the attack success rate increases by 47.6\% when the number of labels rises from 3000 to 9000, demonstrating the importance of adequate feedback for effective adversary learning. In contrast, SA-RL and PA-AD exhibit poor performance even with sufficient feedback, with PA-AD failing entirely in the manipulation scenario. This is likely due to the limited exploration space in these methods, constrained by the fixed victim policy. In contrast, \ourmethod enables better exploration by incorporating an intention policy, allowing for more dynamic interactions and improved performance in complex tasks.

\begin{table}[th!]
    \centering
    \caption{Success rate of different methods with varying numbers of preference labels on the Drawer Open task in the manipulation scenario and the Faucet Close task in the opposite behavior scenario. The success rate is reported as the mean and standard deviation over 30 episodes.}
    \vspace{-0.5em}
    \resizebox{0.75\columnwidth}{!}{
    \begin{tabular}{*{4}{c}c}
        \toprule
        \textbf{Environment}  &\textbf{Feedback}  &\textbf{\ourmethod (ours)} &\textbf{PA-AD}  &\textbf{SA-RL} \\
        \midrule
        \multirow{5}{*}{\specialcell{\textbf{Drawer Open} \\ (manipulation)}}
            & $3000$
            & $65.7\% \pm 37.1\% $    
            & $0.0\%   \pm 0.0\% $ 
            & $8.3\%   \pm 13.2\% $\\
            \cmidrule(l){2-5}
            & $5000$
            & $86.7\%  \pm 18.1\% $    
            & $0.0\%   \pm 0.0\% $ 
            & $21.3\%  \pm 18.9\% $\\
            \cmidrule(l){2-5}
            & $7000$
            & $95.7\%  \pm 13.6\% $    
            & $0.0\%   \pm 0.0\% $ 
            & $28.0\%  \pm 28.1\% $\\
            \cmidrule(l){2-5}
            & $9000$
            & $97.0\%  \pm 6.9\% $    
            & $0.0\%   \pm 0.0\% $ 
            & $13.0\%  \pm 18.5\% $\\
        \midrule\midrule
        \multirow{5}{*}{\specialcell{\textbf{Faucet Close} \\ (opposite behavior)}}
            & $1000$
            & $69.7\%  \pm 35.2\% $    
            & $16.7\%  \pm 9.4\% $ 
            & $2.0\%  \pm 6.0\% $\\
            \cmidrule(l){2-5}
            & $3000$
            & $79.0\%  \pm 16.2\% $    
            & $29.0\%  \pm 14.0\% $ 
            & $6.0\%  \pm 11.7\% $\\
            \cmidrule(l){2-5}
            & $5000$
            & $95.3\%  \pm 9.2\% $    
            & $21.3\%  \pm 12.8\% $ 
            & $3.3\%  \pm 12.7\% $\\
            \cmidrule(l){2-5}
            & $7000$
            & $95.3\%  \pm 7.6\% $    
            & $22.7\%  \pm 12.4\% $ 
            & $4.0\%  \pm 7.1\% $\\
        \bottomrule
    \end{tabular}
    }
    \label{tab:abla_feedback}
\end{table}

\begin{figure*}[!h]
\centering
\begin{center}
\includegraphics[width=0.4\linewidth]{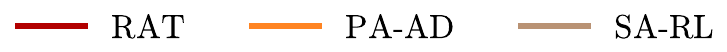}
\vspace{-1em}
\end{center}
\begin{tabular}{ccccc}
\hspace*{-0.8em}
\rotatebox{90}{\qquad Drawer Open}
& \hspace*{-1.2em} \subfloat[$\epsilon=0.05$]{\includegraphics[width=0.25\linewidth]{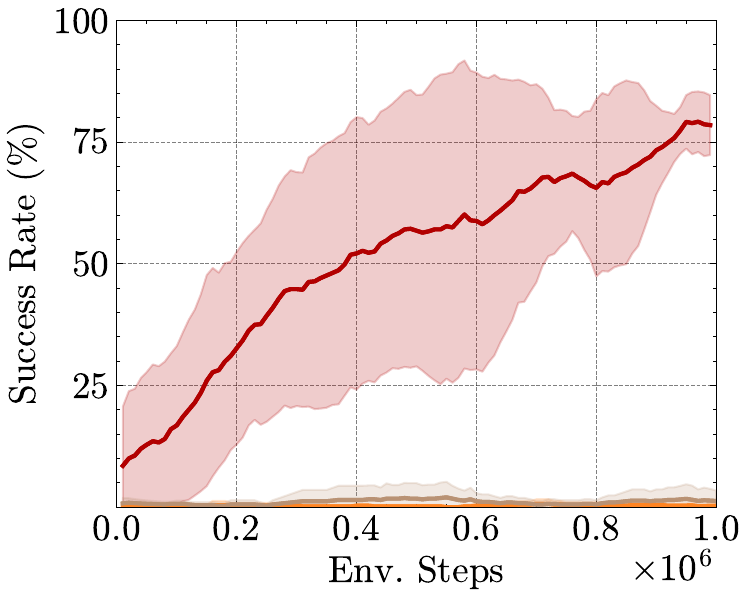}}
& \hspace*{-1.65em} \subfloat[$\epsilon=0.075$]{\includegraphics[width=0.25\linewidth]{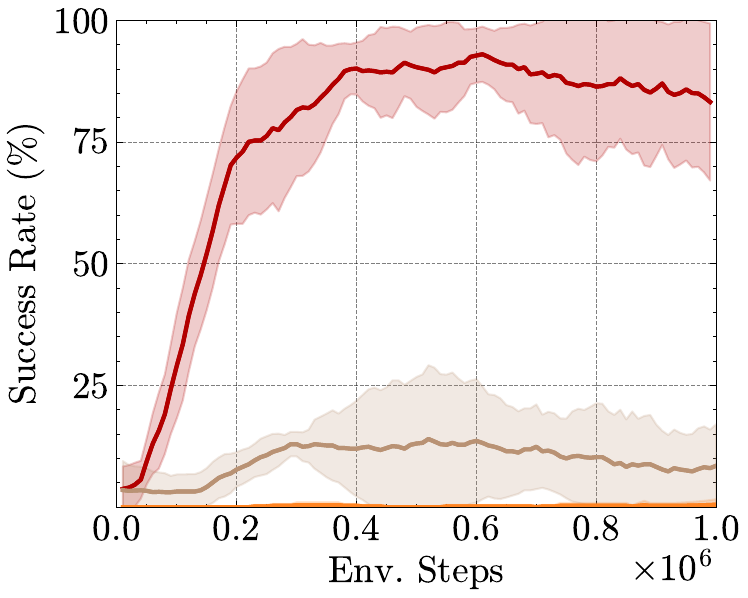}}
& \hspace*{-1.65em} \subfloat[$\epsilon=0.1$]{\includegraphics[width=0.25\linewidth]{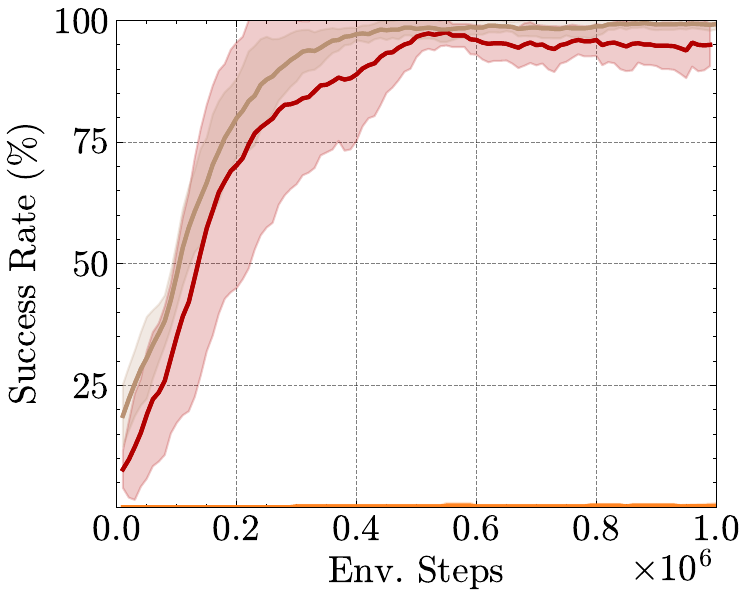}}
& \hspace*{-1.65em} \subfloat[$\epsilon=0.15$]{\includegraphics[width=0.25\linewidth]{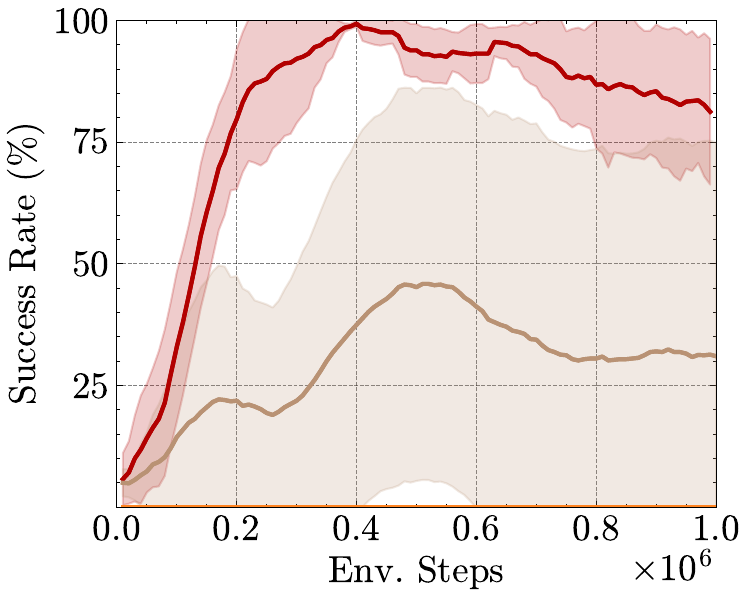}}
\vspace*{-0.7em} 
\\
\hspace*{-0.8em}
\rotatebox{90}{\qquad Faucet Close}
& \hspace*{-1.25em} \subfloat[$\epsilon=0.02$]{\includegraphics[width=0.25\linewidth]{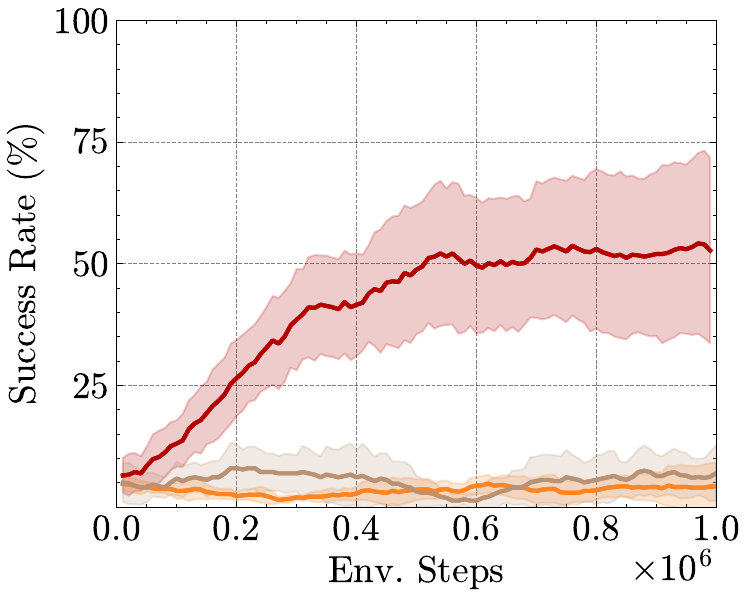}}
& \hspace*{-1.65em} \subfloat[$\epsilon=0.05$]{\includegraphics[width=0.25\linewidth]{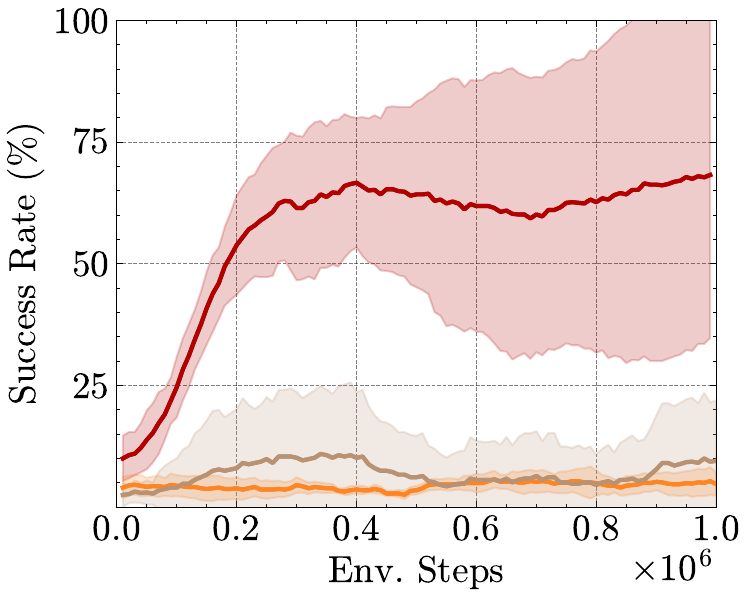}}
& \hspace*{-1.65em} \subfloat[$\epsilon=0.075$]{\includegraphics[width=0.25\linewidth]{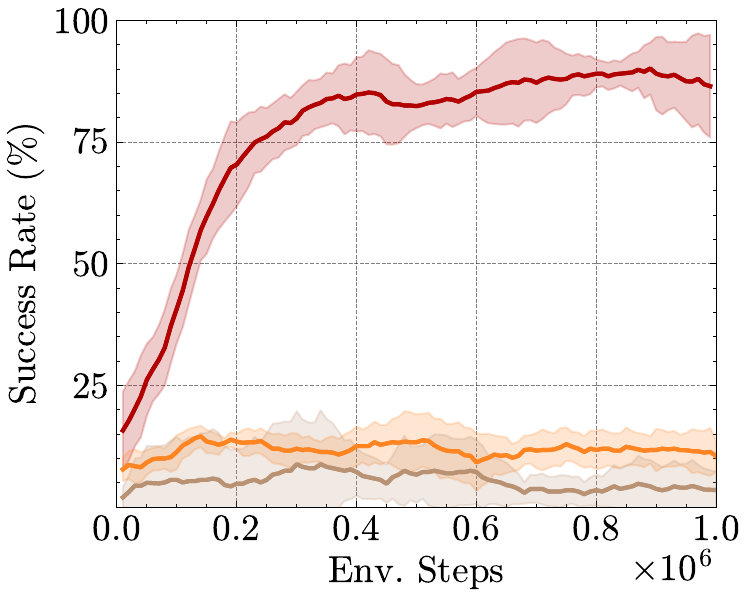}}
& \hspace*{-1.65em} \subfloat[$\epsilon=0.1$]{\includegraphics[width=0.25\linewidth]{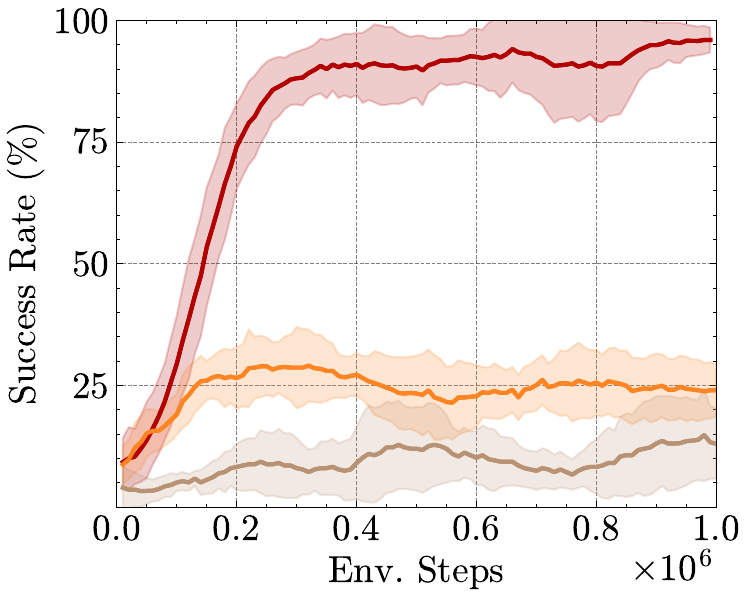}}
\end{tabular}
\caption{Training curves of success rate with different adversarial budgets on Drawer Open for the manipulation scenario and Faucet Close for the opposite behavior scenario. The solid line and shaded area denote the mean and the standard deviation of the success rate across five runs.}
\label{fig:abla_budget}
\end{figure*}

\noindent \textbf{Impact of Different Attack Budgets.}
We also investigate the impact of the attack budget on the performance. To gain further insights, we conduct additional experiments with different attack budgets: ${0.05,0.075,0.1,0.15}$ for the Drawer Open task and ${0.02,0.05,0.075,0.1}$ for the Faucet Close task in the respective scenarios. In Figure~\ref{fig:abla_budget}, we present the performance of the baseline method and \ourmethod with different attack budgets. The experimental results demonstrate that the performance of all methods improves with an increase in the attack budget.

\begin{figure*}[!htp]
\centering
\begin{tabular}{ccccc}
& \hspace*{-1.5em} \subfloat[Faucet Open]{\includegraphics[width=0.4\linewidth, height=0.24\linewidth]{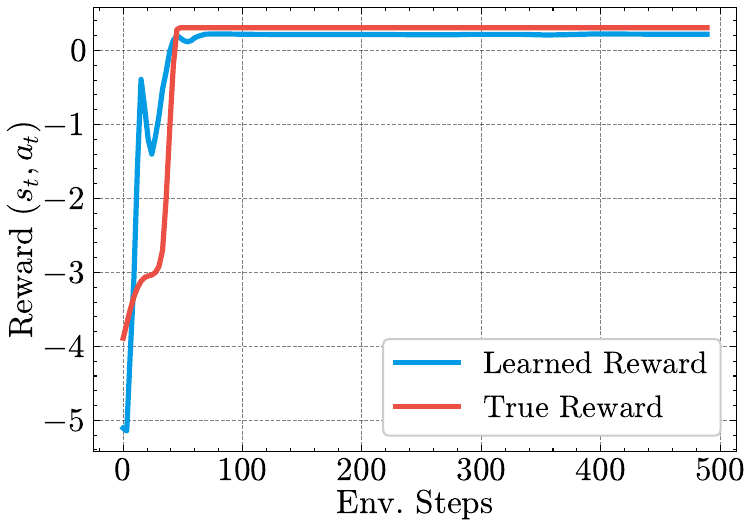}}
& \hspace*{-1.5em} \subfloat[Faucet Close]{\includegraphics[width=0.4\linewidth, height=0.24\linewidth]{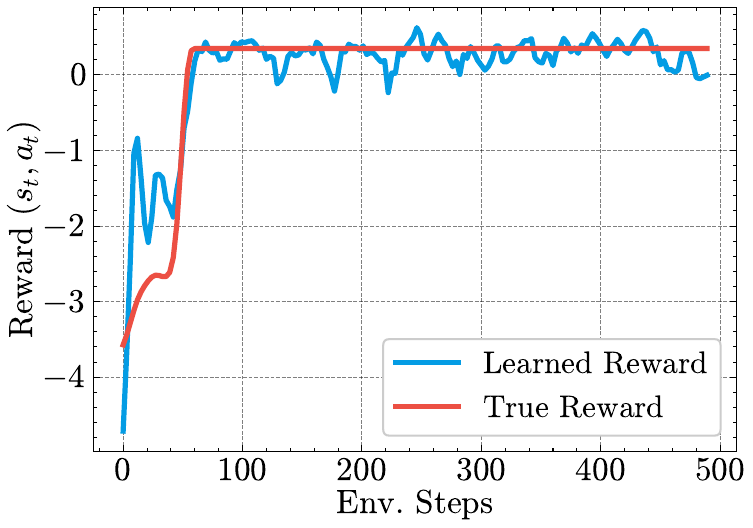}}
\vspace*{-0.7em} 
\\
& \hspace*{-1.5em} \subfloat[Drawer Open]{\includegraphics[width=0.4\linewidth, height=0.24\linewidth]{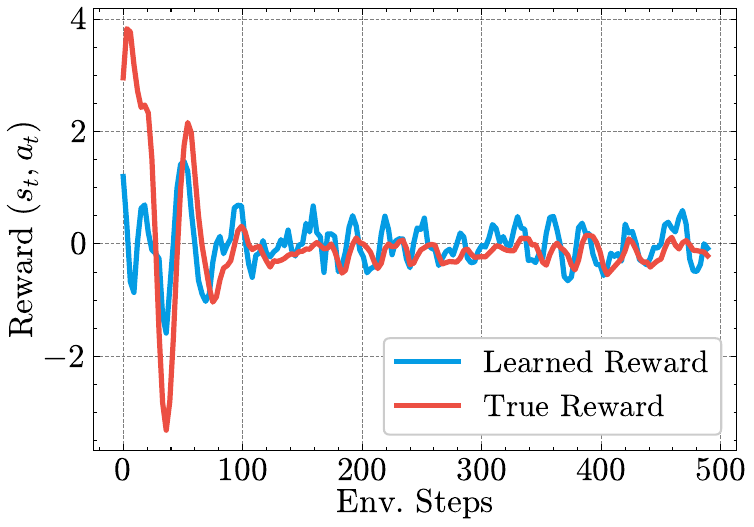}}
& \hspace*{-1.5em} \subfloat[Drawer Close]{\includegraphics[width=0.4\linewidth, height=0.24\linewidth]{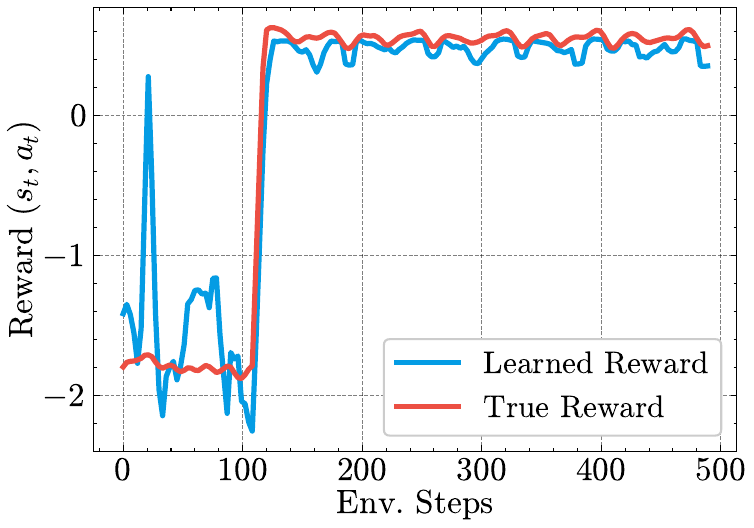}}
\end{tabular}
\vspace{-0.7em}
\caption{\textbf{Quality of learned reward.} Time series of the normalized learned reward (blue) and the ground truth reward (orange). These rewards are obtained from rollouts generated by a policy optimized using \ourmethod.}
\label{fig:reward_quality}
\end{figure*}

\noindent \textbf{Quality of learned reward functions.} 
We further analyze the quality of the reward functions learned by \ourmethod compared to the true reward function. In Figure~\ref{fig:reward_quality}, we present four time series plots that depict the normalized learned reward (blue) and the ground truth reward (orange). These plots represent two scenarios: opposite behaviors and manipulation tasks. The results indicate that the learned reward function aligns well with the true reward function derived from human feedback. This alignment is achieved by capturing various human intentions through the preference data.

\section{Discussion}
In this work, we propose \ourmethod, a novel adversarial attack framework targeting deep reinforcement learning (DRL) agents for inducing specific behaviors. \ourmethod integrates three core components: an intention policy, an adversary, and a weighting function, all trained simultaneously. Unlike prior approaches that rely on predefined target policies, \ourmethod dynamically trains an intention policy aligned with human preferences, offering a flexible and adaptive behavioral target for the adversary. Leveraging advancements in preference-based reinforcement learning (PbRL), the intention policy effectively captures human intent during training. The adversary perturbs the victim agent’s observations, steering the agent toward behaviors specified by the intention policy. To enhance attack efficacy, the weighting function adjusts the state occupancy measure, optimizing the distribution of states encountered during training. This optimization improves both the effectiveness and efficiency of the attack. Through iterative refinement, \ourmethod achieves superior precision in directing the victim agent toward human-desired behaviors compared to existing adversarial attack methods.

An important future direction is the extension of targeted adversarial attacks to LLM-based agents and Vision-Language-Action (VLA) models, which have become increasingly impactful in various practical applications driven by advancements in large-scale models~\cite{zhu2024critical, zhu-etal-2024-pad,jin-etal-2023-parameter,zhu2024utilize,10711229}. Investigating the vulnerabilities of these models to targeted adversarial attacks is crucial for identifying security risks and improving their robustness. These studies can provide valuable insights for designing more resilient architectures and effective defense mechanisms. Additionally, analyzing the behavior of these models under adversarial perturbations in complex real-world scenarios is critical to ensuring their reliability and safety in practical deployments.

\end{document}